%% file: main.tex
\documentclass{article}
\usepackage[preprint]{neurips_2026}
\usepackage[utf8]{inputenc}
\usepackage[T1]{fontenc}
\usepackage{amsmath,amssymb,amsthm}
\usepackage{graphicx,booktabs,array,longtable,tabularx,calc}
\usepackage{microtype}
\usepackage{float,needspace}
\usepackage{placeins}
\usepackage{url}
\usepackage[colorlinks=true,linkcolor=blue,citecolor=blue,urlcolor=blue]{hyperref}
\setcitestyle{authoryear,round}
\newtheorem{proposition}{Proposition}

\newtheorem{corollary}{Corollary}
\theoremstyle{definition}

\newcommand{\E}{\mathbb E}

\title{The Value of Memory in Language Agents}
\author{%
\begin{tabular}{@{}ccc@{}}
Zehao Lin$^{1}$ & Erenkai Zhu$^{2}$ & Helan Hu$^{1}$ \\
{\normalfont\small\texttt{linzh@memtensor.cn}} & {\normalfont\small\texttt{erz21@sjtu.edu.cn}} & {\normalfont\small\texttt{huhl@memtensor.cn}} \\[0.6em]
Xixuan Hao$^{1,3}$ & Chunyu Li$^{1}$ & Linyun Liu$^{1}$ \\
{\normalfont\small\texttt{xhao390@connect.hkust-gz.edu.cn}} & {\normalfont\small\texttt{licy@memtensor.cn}} & {\normalfont\small\texttt{liuly@memtensor.cn}} \\[0.6em]
Kai Chen$^{1}$ & Chenxu Du$^{1}$ & Feiyu Xiong$^{1}$ \\
{\normalfont\small\texttt{chenk@memtensor.cn}} & {\normalfont\small\texttt{ducx@memtensor.cn}} & {\normalfont\small\texttt{xiongfy@memtensor.cn}}
\end{tabular}\\[0.6em]
{\normalfont$^{1}$MemTensor\qquad $^{2}$Shanghai Jiaotong University}\\
{\normalfont$^{3}$The Hong Kong University of Science and Technology (Guangzhou)}
}
\hypersetup{pdfauthor={Zehao Lin; Erenkai Zhu; Helan Hu; Xixuan Hao; Chunyu Li; Linyun Liu; Kai Chen; Chenxu Du; Feiyu Xiong},pdftitle={The Value of Memory in Language Agents}}
\begin{document}
\raggedbottom
\maketitle
\begin{abstract}
\input{sections/00-abstract}
\end{abstract}
\input{sections/body}
\begingroup
\small
\setlength{\bibsep}{0pt}
\interlinepenalty=10000
\bibliographystyle{plainnat}
\bibliography{references}
\endgroup
\bigskip
\appendix
\input{sections/appendix}
\end{document}

%% file: sections/00-abstract.tex
Language agents carry information from past experience into future tasks whose demands may be unknown or change over time. We develop a decision-theoretic account of the capability available in history, retained in storage and realized through use, together with the costs of recovery and revision. Drawing on comparison of statistical experiments, information acquisition and the value of computation, we express these values in units of avoided task loss. Recovery and revision define attainable risk--cost profiles and preservation across uses. A cost-aware simulation condition and five minimal counterexamples establish how information timing, access, computation, feedback and demand shape these comparisons; a pointer with zero answer mutual information can reduce the cost of exact recovery. Allocation and recovery criteria identify when gains pay for estimation, selection and evidence access. Complete-split studies on LoCoMo, LongMemEval-S and PersonaMem cover 3,075 questions and two readers. They show complementary value from saved state and recovered evidence, task-dependent performance orderings, and the role of reuse in amortizing construction costs. Multi-project workflows connect correction to preservation and validation costs. Together, the theory and studies support decisions about what to retain, which recovery options to preserve, and how to adapt memory as evidence and objectives evolve.

%% file: sections/body.tex
\section{Memory as a resource for future decisions}\label{sec:intro}

Language agents operate across conversations, projects and changing objectives. Their memory carries decisions, plans, evidence and constraints from earlier work into requests that may not yet be known. Consider three linked projects: a completed implementation in Quince, an audit receipt in Rue and a joint release in Verbena. A new acceptance requirement may call for another audit while the implementation remains reusable. Continuing the work requires the agent to preserve those distinctions, recover their support when needed, and revise the affected plan. What matters is the future work that the history makes possible.

This raises a valuation problem. How much decision value was available in the original history? How much survives in the stored state, and how much can a bounded agent actually use? When future demands change, what additional value can be recovered or created through revision, at what cost, and with what effects on other uses? These questions connect memory capacity, retrieval, continual updating and evaluation.

\textbf{Our thesis is that the value of memory lies in its contribution to future decisions, jointly with the cost of preserving, accessing and revising that contribution.} A stored state can carry relevant information that its reader fails to use. An address can accelerate recovery while carrying no direct information about the answer. A revision can improve one task and damage another. A useful account of memory value must describe these possibilities together.

We develop a decision-theoretic framework that connects memory's value across storage, use, recovery and revision (Figure~\ref{fig:framework}). This connection makes compression, retrieval and updating analyzable as ways of allocating decision capability across time and uses. Our contribution has three parts. First, we express original, retained and realized value in a common unit of avoided task loss, and relate the gap between retained and realized value to the reader's implementation regret at both memory and baseline states. Second, we connect these values to resource use through recovery and revision frontiers, cost-aware simulation and explicit models of unknown demand and interactive feedback. Third, we derive decision criteria for design and measurement: allocation regret connects estimated value to selection cost, a recovery-option identity prices access to later evidence, and preservation conditions guide revision of dependent plans. Operational risk bounds connect these decisions to what can be measured.

Five minimal counterexamples establish why information timing, access, component interaction, feedback and demand can change value comparisons. The behavioral studies follow these distinctions into operation. Complete splits of LoCoMo, LongMemEval-S and PersonaMem connect stored evidence to realized answers and paid recovery across 3,075 questions and two readers. Multi-project workflows expose the value restored by revision, the effects on other uses and the resources consumed along the way. Measurement controls distinguish changes in these outcomes from changes in their assessment. Together, the studies make the value account concrete across retention, use, recovery and adaptation.

\begin{figure}[t]
\centering
\includegraphics[width=\linewidth]{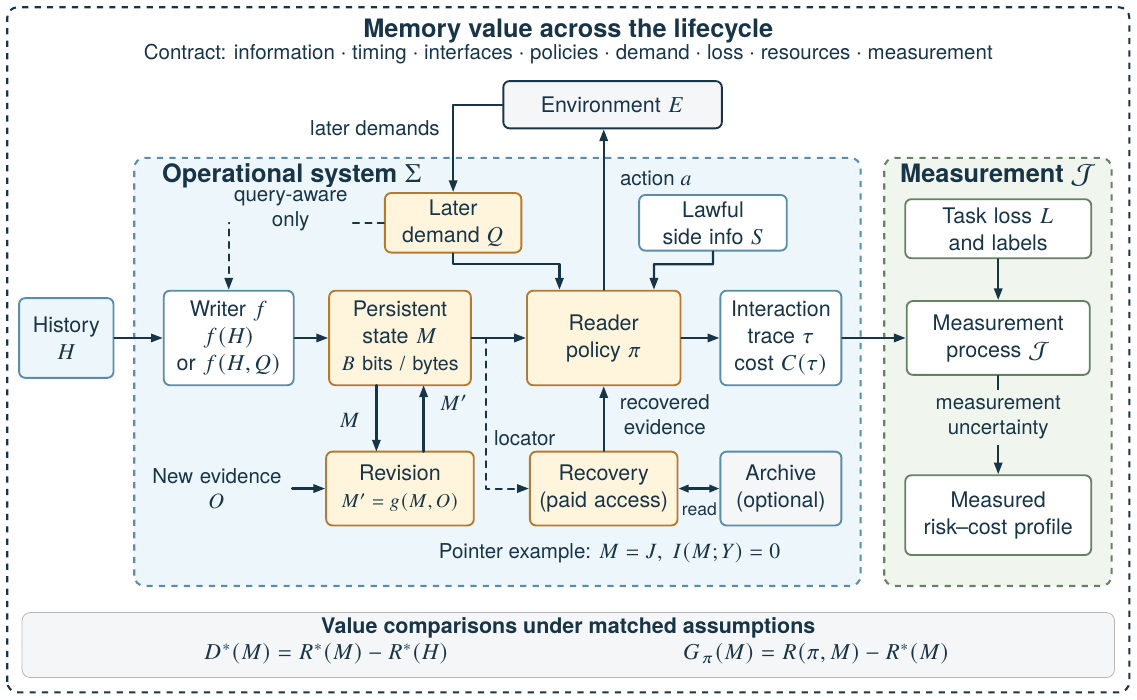}
\caption{Memory value across information, operation and measurement. The operational system \(\Sigma\) includes writing, storage, recovery, revision and reading. Its actions affect the environment and later demands; \(\mathcal J\) measures the resulting trajectories under a task loss and reference labels. The outer boundary records comparison conditions. The dashed query input denotes the query-aware alternative. The lower strip compares potential risks under matched assumptions; \(O\) denotes new revision evidence.}
\label{fig:framework}
\end{figure}

\section{Foundations of memory value}\label{sec:related}

Classical comparison of statistical experiments provides the starting point: information has value relative to a decision problem and the actions it enables \citep{blackwell1951comparison}. Simulation explains why a decision maker with access to richer information can preserve the performance of a less informed one. For an implemented memory system, the relevant question is whether that simulation respects its information timing, computational constraints and access costs.

Several established lines of work make these restrictions explicit. Predictive \(\mathcal V\)-information measures information through an allowed prediction family and optimized log loss \citep{xu2020usable}. Action-dependent side information connects information acquisition to its cost \citep{weissman2009side}. Prompt-compression theory distinguishes query-aware and query-agnostic encoders \citep{nagle2024compression}. These foundations motivate separate accounting for stored content, feasible computation and paid access.

Memory architectures make these choices concrete. Memory$^3$ formulates memory-format choice as writing cost plus expected reuse-weighted reading cost, and implements explicit memory as sparse key-values \citep{yang2024memory3}. Metis and MemT-1.0 maintain native memory states that evolve within model computation \citep{zhang2026metis,memtensor2026memt1}. MemOS organizes plaintext, activation and parameter memories through lifecycle operations and MemCubes carrying provenance and version metadata \citep{li2025memos}.

At the agent level, Generative Agents connects recorded experience and reflection to planning \citep{park2023generative}; MemGPT manages transfers across memory tiers \citep{packer2023memgpt}; A-Mem evolves linked notes \citep{xu2025amem}; and Mem0 extracts, consolidates and retrieves conversational memories \citep{chhikara2025mem0}. CoALA organizes memory modules, internal and external actions, and a recurring decision process \citep{sumers2024coala}. These are architectural precedents for the components whose retained and realized decision value we study under explicit information and resource conditions.

Recent memory research gives concrete evaluation models for these questions. DeMem formulates decision distortion and an implementation-error bound with a query-aware encoder \citep{zou2026demem}. \textsc{MemAudit} defines a finite package optimum under representation, evidence and cost choices \citep{bhargava2026memaudit}. VitaBench evaluates lifecycle decisions over a horizon and separately studies a rollback extension \citep{yu2026vitabench}; MemOps represents lifecycle operations, target binding and state transitions \citep{hao2026memops}. Update-sufficiency auditing distinguishes content, delivery and scoring conventions \citep{zhang2026correct}. These studies specify particular stages or uses of memory. Transferring conclusions between their settings requires the comparison conditions in Section~\ref{sec:contract}.

A cross-layer compaction perspective connects KV caches, prompts, bounded states and agent stores through rate--distortion trade-offs \citep{colaco2026compaction}. Session-handover theory uses statistical sufficiency and comparison of experiments to distinguish capacity, writer and continuation effects, including the cost of committing before a future query is known \citep{kato2026handover}. These results provide direct foundations for following value through storage, use and later information acquisition.

Preservation and validity also have direct precedents. WISE evaluates unrelated behavior after edits \citep{wang2024wise}. A consolidation study finds that its proposed memory-quality score fails to establish prediction of external transfer performance \citep{annapureddy2026epistemics}. Partial identification formalizes values left undetermined by evidence \citep{manski2024identification}, while performative prediction models decisions that change the future evaluation distribution \citep{perdomo2020performative}. These perspectives bring preservation, uncertainty and feedback into the same value account.

The design questions concern which computations and future options to support. Rational metareasoning supplies a decision-theoretic account of paid computation \citep{hay2012computations}, and irreversible-decision theory values options retained until later information arrives \citep{arrow1974irreversibility}. For revision and reuse, truth maintenance and reliable build models connect valid results to their justifications and dependency records \citep{doyle1979truth,coetzee2012build}. Translation validation complements dependency tracking by checking each compiler output against its source specification \citep{pnueli1998translation}. Together, these foundations guide the constructive directions in Section~\ref{sec:implications}.

\section{From history to retained and realized value}\label{sec:value}

\subsection{Potential, retained and realized value}\label{sec:value-levels}

Let \(H\) be the writer's history, \(M=f(H)\) the persistent state, \(Q\) a later demand, \(Y\) the task-relevant world state, and \(S\) the reader's lawful side information. For demand law \(P\), loss \(\ell\), and policy class \(\Pi\), define
\begin{equation}
R^*(M;P,\ell,\Pi)
=\inf_{\pi\in\Pi}\E_P\!\left[\ell\!\left(Y,Q,\pi(M,S,Q)\right)\right].
\label{eq:potential}
\end{equation}
The policy class specifies whose decisions the representation can support. An unrestricted class gives a potential-information reference; a constrained class gives a notion of usable information. A particular deployed reader has its own risk \(R(\pi,M)\).

Under common side information and policy permissions, incremental value relative to an empty memory is
\begin{equation}
V^*(M)=R^*(\varnothing)-R^*(M).
\label{eq:value}
\end{equation}
The same baseline defines original information value \(V^*(H)\). The value lost between history and stored state is their decision-risk difference:
\begin{equation}
\begin{aligned}
D^*(M)&=R^*(M)-R^*(H),\\
V^*(H)&=V^*(M)+D^*(M).
\end{aligned}
\label{eq:value-accounting}
\end{equation}
All three quantities use the common loss unit. Under the simulation and optional-ignorance conditions in Appendix~\ref{app:information}, original value is at least retained value, and both are nonnegative. These conditions specify when the history can reproduce memory-based decisions within the allowed resources.

An implemented reader realizes a further, distinct comparison:
\[
V_\pi(M)=R(\pi,\varnothing)-R(\pi,M).
\]
This effect can be positive, zero or negative because the fixed reader may use the state imperfectly. Thus potential value, actual performance and resource use provide complementary design information. A full-history LLM measures \(R(\pi,H)\); a selected evidence packet tests use with that support supplied; a package optimum ranges over its declared candidates.

For an admissible reader, define its implementation regret as
\(G_\pi(M)=R(\pi,M)-R^*(M)\geq0\). The incremental-value difference also depends on how well that reader uses the empty baseline:
\begin{equation}
V^*(M)-V_\pi(M)=G_\pi(M)-G_\pi(\varnothing).
\label{eq:regret-gap}
\end{equation}
Thus \(G_\pi(M)\) measures the unrealized capability at \(M\), while the value difference compares two implementation gaps. Appendix~\ref{app:operational} gives an explicit case in which that difference is negative.

\input{sections/worked-value-example}

\subsection{The system that uses the representation}

The operational system \(\Sigma\) includes writing, storage, indexing, recovery, revision, reading and failure handling. In a finite-horizon environment \(E\), it induces a trajectory \(\tau\), risk and cost:
\begin{equation}
R(\Sigma;E)=\E_{P_{\Sigma,E}}[L(\tau)],
\qquad C(\tau)\in\mathbb R_+^d.
\end{equation}
The trajectory includes user clarifications, failed calls, retries and delivered actions. Its distribution can depend on the system: an omitted constraint may cause a follow-up request, and a larger context may increase latency or delivery failure.

A marginal memory comparison specifies the intervention, for example \(\Sigma(M_1,K)\) versus \(\Sigma(M_0,K)\), with a common reader, access policy and environment rule \(K\). This supports a conditional effect claim. Attribution to storage, retrieval or reading requires additional controls because these components can interact.

\subsection{Comparison conditions}\label{sec:contract}

Table~\ref{tab:contract} states the elements that fix the meaning of a comparison. The contract records both what is held common and what is deliberately changed. A study can isolate one element while retaining a clear account of the remaining operational conditions.

\begin{table}[H]
\centering\small
\caption{A comparison contract connects an evaluation result to its interpretation.}
\label{tab:contract}
\begin{tabularx}{\linewidth}{@{}p{0.33\linewidth}X@{}}
\toprule
Component & Question resolved \\
\midrule
Information and revelation order & What can the writer know before committing to memory? \\
Persistent state and side information & Which stored or externally supplied resource supports the answer? \\
Recovery and update interfaces & What can be reread, refreshed, written back or asked of the user? \\
Policies and computation & Which decisions are feasible under the available reader and tools? \\
Demand and environment & Which future uses and feedback processes enter the evaluation? \\
Consequences and loss & What makes an answer, action, delay or failed delivery costly? \\
Resource accounting & What is charged, at which stage, and in which units? \\
Measurement & Which labels are verified, ambiguous or unresolved? \\
\bottomrule
\end{tabularx}
\end{table}

These conditions keep historical evidence, stored state, reading policy and recovery resources visible in the same value account. They define the feasible simulations in Section~\ref{sec:counterexamples} and the observed comparisons in Section~\ref{sec:evidence}.

\section{Recovery and revision under evolving demands}\label{sec:revision}

Recovery and revision extend the decisions an agent can support over time. Reading an archive exposes historical information through a paid interface, while checking a live resource, a changed permission or a new user requirement can supply observations beyond the old history. Recording each observation's source and time therefore distinguishes restored historical support from adaptation to an evolving world.

For a persistent revision \(M'=g(M,O)\), a useful outcome is
\[
\bigl(R_T(M'),\ R_U(M')-R_U(M),\ c\bigr),
\]
where \(T\) is the target use, \(U\) a finite set of other uses, and \(c\) the revision cost. We aggregate the other uses as \(R_U(M)=\sum_{u\in U}w_uR_u(M)\), with \(w_u\geq0\) and \(\sum_{u\in U}w_u=1\). The weights and each use's demand law and loss are fixed before comparison, with losses expressed in a common unit. Fix a cost functional \(g\mapsto C_g\), such as expected resource use or a componentwise almost-sure budget bound. For a repair family \(\mathcal G\), target-risk threshold \(\epsilon\), and preservation tolerance \(\eta\), define
\begin{equation}
\begin{split}
\mathcal K(\epsilon,\eta)=\{\,c:\ \exists g\in\mathcal G,\quad
&R_T(M'_g)\leq\epsilon,\\[-2pt]
&R_U(M'_g)-R_U(M)\leq\eta,\quad C_g\preceq c\,\}.
\end{split}
\label{eq:frontier}
\end{equation}
The nondominated boundary of this set is the correction-cost frontier. Per-use preservation replaces the aggregate constraint by \(R_u(M'_g)-R_u(M)\leq\eta_u\) for every \(u\in U\). A robust variant requires the risk constraints to hold uniformly over a specified family of demand laws. Tested strategies establish attainable points; the full frontier depends on the available policies and interfaces.

The plan example in Section~\ref{sec:worked-value} illustrates these constraints. Suppose Rue's receipt has been recovered and a later source notice invalidates it, making the current predicate $B^\dagger=0$. Grading the unmodified stored plan against the new sources gives expected loss $1/4$. Updating the receipt predicate while retaining Quince's completion yields the correct plan $(A,0,0)$ and zero loss. Resetting both project predicates also resolves the audit and joint-release decisions, but unnecessarily repeats Quince's completed work in half the histories, adding $1/8$ portfolio loss. The first repair preserves Quince's use; the second violates zero per-use tolerance. The operation being valued therefore includes both restored target decisions and the capabilities carried forward.

The units follow the underlying quantities. Fixed-length capacity is measured in bits; implementations can report bytes and their encoding. Decision value is measured in avoided loss, with error probability as one possible loss. Archive reads, model tokens, time and user interactions remain separate cost coordinates until an explicit pricing or preference rule combines them. This makes a capacity--accuracy curve a statement about a specified writer, reader, demand family and recovery allowance.

Unknown future demand changes the valuation itself. For an admissible set of exogenous demand laws \(\mathcal P\), retain the profile \(P\mapsto V^*(M;P)\); choosing a mixture or a worst-case criterion adds a decision preference. When the agent's actions shape later requests, the operational model instead uses the induced trajectory law \(P_{\Sigma,E}\) under a common environment rule \(E\). Both forms of uncertainty affect which capabilities are worth preserving now and which can be reacquired later.

Failures also affect attainable value. With delivery-failure probability \(p_F\), delivered risk \(r\), and specified failure loss \(\rho\),
\[
R=(1-p_F)r+p_F\rho.
\]
Retries and auxiliary computation contribute to the same trajectory's cost, bringing the method's decision benefit and operational burden into a common account.

\section{How information, access and demand shape value}\label{sec:counterexamples}

Policy simulation identifies when a change of representation or interface preserves attainable decision value. A comparison contract collects its information timing, interfaces, feasible policies, loss and resource conditions (Table~\ref{tab:contract}). Consider contracts \(\mathcal C_0\) and \(\mathcal C_1\) over the same environment family. Let \(Z_i\) be the externally consequential interaction trace and \(C_i\) its cost vector. Feasible policies respect a hard budget almost surely in every allowed environment.

\Needspace{12\baselineskip}
\begin{proposition}[Cost-aware simulation]\label{prop:simulation}
Suppose every budget-\(b\) policy under \(\mathcal C_0\) has a causal simulator under \(\mathcal C_1\) that, in every allowed environment, preserves the distribution of \(Z_0\) and has cost at most \(C_0+\delta\) under a coupling. For nonempty feasible policy classes and a common bounded loss,
\[
R^*_{\mathcal C_1}(b+\delta;P,L)\leq R^*_{\mathcal C_0}(b;P,L)
\]
for every declared distribution \(P\) over environments.
\end{proposition}
\begin{proof}
Fix \(P\) and \(\epsilon>0\). Choose a feasible policy under \(\mathcal C_0\) with risk within \(\epsilon\) of its infimum. Its simulator is feasible at budget \(b+\delta\) and has the same risk under \(P\). The infimum under \(\mathcal C_1\) is therefore at most \(R^*_{\mathcal C_0}(b;P,L)+\epsilon\). Letting \(\epsilon\downarrow0\) proves the claim.
\end{proof}

\begin{corollary}[Robust comparison]\label{cor:robust}
Under Proposition~\ref{prop:simulation}, the same causal simulation condition implies
\[
\inf_{\pi_1\in\Pi_1(b+\delta)}\sup_{E\in\mathcal E}R(\pi_1;E)
\leq
\inf_{\pi_0\in\Pi_0(b)}\sup_{E\in\mathcal E}R(\pi_0;E).
\]
\end{corollary}
\begin{proof}
Each old policy has one simulator preserving its risk in every environment, hence preserving its worst-case risk. Taking infima gives the inequality.
\end{proof}
This conclusion requires no single future-demand mixture. The worst-case objective is an explicit decision criterion, and the simulator must work across the environment family.

The condition extends the classical simulation argument to explicit access and conversion costs. A free reversible representation is a special case. A simulator that must reread an archive or ask the user needs an allowance for that interaction. The following constructions isolate five structural dependencies of memory value.

\subsection{Query timing changes required capacity}\label{sec:timing}

Let \(H=(B_1,\ldots,B_n)\) contain independent fair bits, and let a later query request one coordinate. With no recovery, zero error for every history and query requires \(n\) fixed-length bits when the writer commits before seeing the query. Any two distinct histories differ at a coordinate that a future query can request. Merging those histories therefore loses a required distinction.

If the writer sees the query and can reread \(H\) before encoding, one answer bit suffices (Figure~\ref{fig:comparisons}a). For \(n=8\), changing revelation order changes the requirement by seven bits. That gap depends on the answers distinguished by the query family; duplicating query names for the same coordinate leaves the requirement at one bit. This coordinate-query construction also appears in session-handover analysis, where it yields a log-loss lower bound \citep{kato2026handover}. Comparing runtime compression with persistent writing therefore requires matching when the query becomes visible.

\subsection{A pointer can have value with zero answer information}\label{sec:pointer}

An archive has \(N\) opaque addresses holding independent fair bits \(B_1,\ldots,B_N\). The target address \(J\) is uniform and independent of the bits, and the answer is \(Y=B_J\). A paid read returns the selected bit and an indicator of whether its address is the target.

A retained pointer \(M=J\) satisfies \(I(M;Y)=0\): knowing the address leaves the answer bit fair. Yet it enables exact recovery in one read. Without the pointer, the target is found by search; zero error requires up to \(N\) reads and an expected \((N+1)/2\) reads when stopping at the first hit. At \(N=4\),
\[
I(M;Y)=0,
\qquad
\E[C_{\rm search}]=\tfrac52,
\qquad
\E[C_{\rm pointer}]=1.
\]
Under a one-read budget, risk falls from \(3/8\) to zero (Figure~\ref{fig:comparisons}b).

Thus answer mutual information alone cannot order the value of memory used through a paid retrieval interface. The address carries information about \emph{where to acquire} the answer. The construction also distinguishes persistent bits from acquisition cost: the pointer occupies \(\lceil\log_2N\rceil\) bits, while reads have their own unit. Appendix~\ref{app:pointer} gives the risk--cost curve and access assumptions.

\subsection{Memory and reader gains interact}\label{sec:interaction}

For a fair target bit, compare an empty state with a state containing the answer, and a constant-zero reader with a capable reader. The four risks are
\[
\begin{array}{c|cc}
&\text{constant reader}&\text{capable reader}\\ \hline
\text{empty memory}&1/2&1/2\\
\text{answer in memory}&1/2&0
\end{array}.
\]
The joint improvement is \(1/2\). Improving memory first produces zero immediate gain and assigns the subsequent gain to the reader. Improving the reader first assigns the subsequent gain to memory. Both paths agree with the full intervention table (Figure~\ref{fig:comparisons}c).

The example identifies complementarity while leaving the division of joint credit dependent on an attribution convention. A conditional memory effect can be measured with a fixed reader; a claim about the fraction of loss caused by storage requires more than the aggregate improvement.

\subsection{Feedback can reverse the ordering left unresolved by replay}\label{sec:feedback}

Policy \(i\in\{0,1\}\) opens with action \(a=i\) and answers the subsequent task with \(i\). The correct answer is the task type \(q\). Forced uniform replay gives both policies risk \(1/2\). In interaction, let \(p_a=\Pr(q=1\mid a)\). Two possible feedback environments give:
\begin{table}[H]
\centering\small
\caption{Identical forced-replay behavior is compatible with opposite interactive rankings.}
\label{tab:feedback}
\begin{tabular}{@{}lrrrr@{}}
\toprule
Environment & \(p_0\) & \(p_1\) & Policy 0 risk & Policy 1 risk \\
\midrule
\(E_+\) & \(0\) & \(1/4\) & \(0\) & \(3/4\) \\
\(E_-\) & \(3/4\) & \(1\) & \(3/4\) & \(0\) \\
\bottomrule
\end{tabular}
\end{table}

The forced-replay response table is identical in the two environments, but the interactive ranking reverses (Figure~\ref{fig:comparisons}d). Replay measures capability on supplied follow-ups. Estimating avoided clarification or user effort additionally requires the process by which those follow-ups arise.

\subsection{Unknown future demand permits several rankings}\label{sec:demand}

Suppose \(M_A\) stores a fair bit \(A\), \(M_B\) stores an independent fair bit \(B\), and a future query asks for \(A\) with probability \(p\). Their risks are
\[
R_A(p)=\frac{1-p}{2},
\qquad R_B(p)=\frac p2.
\]
Knowing only \(p\in[0.2,0.8]\) identifies a relative-value set
\(\{R_B(p)-R_A(p)\}=[-0.3,0.3]\).
Each memory is preferable under some admissible demand laws. A risk profile preserves that information. A unique choice requires an additional commitment, such as a specified mixture or a worst-case decision criterion. Demand uncertainty therefore belongs in the value definition.

\begin{figure}[t]
\centering
\includegraphics[width=\linewidth]{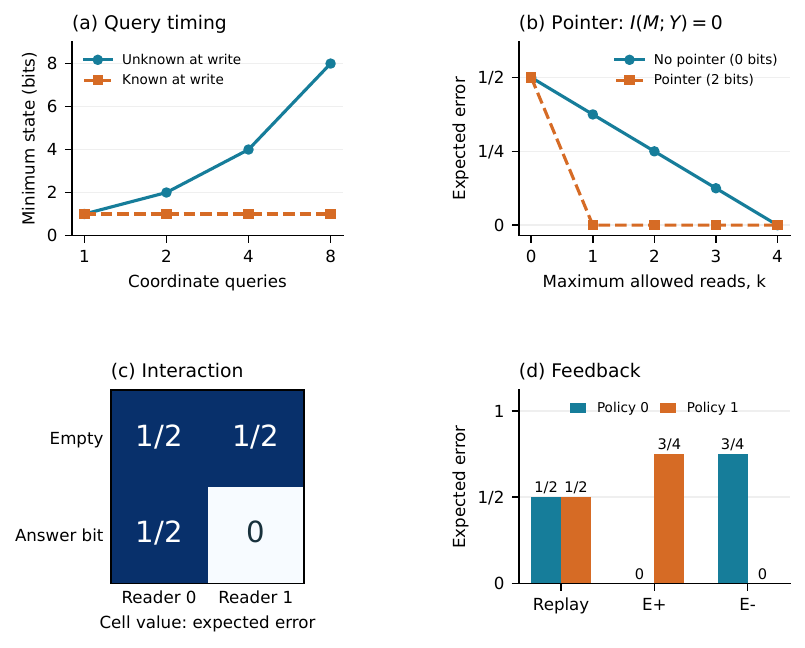}
\caption{Four mechanisms shape memory value. (a) Query timing changes required state size. (b) A two-bit pointer with \(I(M;Y)=0\) reduces risk under paid recovery. The horizontal axis is the maximum allowed reads \(k\); zero-error expected reads are \(5/2\) for search and \(1\) with the pointer. (c) Memory and reader improvements interact. (d) Feedback changes the ranking compatible with fixed replay. All values follow the constructions in Section~\ref{sec:counterexamples}.}
\label{fig:comparisons}
\end{figure}

\section{Identifying and measuring memory value}\label{sec:measurement}

Measuring memory value connects three kinds of evidence: the decisions a representation can support, the performance a reader realizes, and the changes in loss and cost produced by recovery or revision. A measurement process $\mathcal J$ maps observed answers and actions to the declared task loss. Operational references then relate those observations to the potential-value account.

\subsection{The measurement process}

The measurement process \(\mathcal J\) observes evidence about \(\tau\) and assigns outcomes. Its errors include rejecting equivalent actions, confusing a proposal with an executed edit, or judging against a target from the wrong time. Such errors introduce uncertainty about system performance.

The task determines whether an output defect belongs to \(\Sigma\) or \(\mathcal J\). When a production consumer requires a machine-readable artifact, delivery in the wrong format has a real consequence. When a scorer rejects an alternative that satisfies the task, the scorer needs revision. Section~\ref{sec:fictional} gives a concrete example in which two operation words produce the same required disposition of completed work.

Measurement uncertainty propagates through the comparison. For \(n\) fixed units, \(w\) known errors and \(u\) unresolved labels under a \([0,1]\)-bounded loss, empirical loss lies in
\([w/n,(w+u)/n]\). The interval describes uncertainty in those labels; population inference requires a sampling model as well. Applying the same distinction to both conditions preserves uncertainty about their value difference.

\subsection{Operational risk and value bounds}
A nonempty finite panel \(\mathcal B\subseteq\Pi\) of admissible readers provides a tractable reference for optimal risk. Its best achievable risk bounds \(R^*(M)\) from above:
\begin{equation}
R^*(M)\ \leq\ U_{\mathcal B}(M)
\equiv\min_{\pi\in\mathcal B}R(\pi,M).
\label{eq:panel-reference}
\end{equation}
Panel membership is determined by lawful evidence and resource budgets, including any archive access or human assistance. Successful evidence-supported reads identify attainable behavior. In the finite plan example, enumerating the permitted actions gives the risks in Table~\ref{tab:worked-value} exactly. Representation limits can be established from histories that induce identical lawful observations but require different answers. These two forms of evidence supply upper and lower risk bounds, respectively; Appendix~\ref{app:operational} develops the constructions and their statistical evaluation.

If simultaneously valid bounds are \(R^*(M)\in[\underline r_M,\overline r_M]\) and \(R^*(\varnothing)\in[\underline r_0,\overline r_0]\), they imply
\begin{equation}
V^*(M)\in
[\underline r_0-\overline r_M,\ \overline r_0-\underline r_M].
\label{eq:value-bounds}
\end{equation}
This interval preserves uncertainty about both the memory and the baseline, following the distinction between identification and sampling uncertainty \citep{manski2024identification}.

Evaluation connects these bounds to future uses through a task loss and a verifier calibrated on task-consistent answers and meaning-preserving variations. A fixed reader shows realized performance, while an admissible reference panel explores what the same evidence can support. Panel selection and evaluation use independent data or a uniform bound covering the selection procedure. Retrieval and revision interventions then reveal how additional resources change performance and preservation across uses. Together, these observations connect the information carried by memory to the actions and costs it produces.

\section{Memory value in operation}\label{sec:evidence}

The empirical studies examine the value realized from retained state, the additional value obtained through recovery, and the costs and side effects of revision. Three public conversational benchmarks supply diverse demands for historical information; multi-project workflows make correction and preservation directly observable.

\input{sections/public-study}

\subsection{A two-project update with visible consequences}\label{sec:fictional}

The fictional world contains Cedar, an audit-and-experiment project, and Birch, a report with one valid citation and one citation requiring an owner proposal. New evidence completes Cedar's audit and advances its plan, while Birch remains unchanged. Two full-state rewrites and two Cedar-only updates save memories under a common 1,024-byte limit. Fresh readers receive a saved state and either a Cedar continuation request or a Birch report request.

Figure~\ref{fig:cedar} shows actual saved excerpts and reader output. All four updated states preserve the tested facts in both projects. The full rewrite can change Birch's wording while preserving its meaning. This makes the control useful for separating textual change, semantic preservation and action acceptance.

\begin{figure}[!htbp]
\centering
\fbox{\begin{minipage}{0.94\linewidth}
\small
\textbf{Task context.} Cedar's audit has just finished. Prepare the next comparison locally, avoid repeating the audit, and check resources before a paid run. Birch's report and citation instructions remain unchanged.

\medskip
\begin{tabularx}{\linewidth}{@{}XX@{}}
\textbf{Full rewrite: saved Cedar excerpt} &
\textbf{Scoped update: saved Cedar excerpt}\\[2pt]
\textquotedblleft As of t5: Cedar audit completed; audit note cedar/audit.md delivered. Plan is now r3. Controlled comparison NOT started.\textquotedblright &
\textquotedblleft As of t5: Cedar audit completed; audit.md delivered. Plan r3. Controlled comparison NOT started.\textquotedblright\\
\end{tabularx}

\medskip
\textbf{Saved reader output, scoped update.}
\[
\begin{array}{l|lll}
\text{field} & \texttt{audit\_work} & \texttt{comparison\_work} & \texttt{paid\_launch}\\ \hline
\text{operation} & \texttt{hold} & \texttt{prepare} & \texttt{hold}
\end{array}
\]
The reader also reports the audit as complete and the comparison as pending. Its reason includes: \textquotedblleft Prepare local comparison design, no repeat audit.\textquotedblright

\textbf{Strict operation check.} Expects \texttt{reuse} for completed audit work, so it rejects this artifact.

\textbf{Consequence check.} Both \texttt{hold} and \texttt{reuse} keep the completed audit off the work queue; local preparation proceeds and paid execution remains gated. The artifact satisfies these task requirements.
\end{minipage}}
\caption{One concrete measurement disagreement. Excerpts come from the saved fictional study; the context sentence summarizes its task. The scoring review changes how the same output is assessed.}
\label{fig:cedar}
\end{figure}

Across 28 reader evaluations spanning updated-state and baseline conditions, the strict operation rule accepts 17 complete artifacts, whereas the consequence-based check accepts 22. All five discrepancies concern the \texttt{hold}/\texttt{reuse} equivalence in Figure~\ref{fig:cedar}. Applying the task's requirement to keep completed work off the queue changes measured acceptance by \(5/28\), or 17.9 percentage points. The task requirements, \(M\), \(\Sigma\) and all saved outputs remain fixed; the change is in \(\mathcal J\)'s interpretation of the operation field.

\subsection{Three longitudinal failure cases}\label{sec:cases}

The Cedar/Birch control makes the action requirements explicit. In ongoing projects, those requirements must also be recovered from an evolving history. The following cases trace one owner's internal software and research projects across sessions and two agent tools. Writers processed a continuous history exceeding six million tokens before the tested follow-ups; revisions were saved as persistent states and evaluated in fresh reader contexts.

\paragraph{Case 1: Avoidable overwriting under a capacity limit.}
A retrieval-assisted revision improved support for a research-planning request but omitted the state needed for another project's report. A reader then treated that report state as unknown. We appended the unchanged report paragraph from the earlier memory and repeated both reads. The augmented state still fit the original 8 KiB limit: report-state readout recovered while the target readout in the inspected lineage was preserved. The intervention establishes a feasible state containing both pieces of support and rules out capacity-forced deletion for this observed revision.

\paragraph{Case 2: Unsupported completion can score as successful recall.}
A rewrite with no new evidence asserted that an earlier revision was complete, although the input state supplied no basis for that completion claim. The statement happened to match the historical label, increasing answer accuracy. Under an answer-only loss this is a correct prediction; under an evidence-grounded continuation contract its support remains absent. Checking the writer's evidence access distinguishes the two interpretations.

\paragraph{Case 3: Correct state readout can lead to an invalid edit.}
A reader recovered current report identifiers but proposed replacing a citation whose existing use was valid. The report contained a separate citation problem in another section. A fact-level readout therefore missed the actionable distinction between a current identifier and an occurrence that requires revision. An artifact-level acceptance check exposed the proposed harm. The required memory includes the relationship among a source, its location and its rationale.

\Needspace{14\baselineskip}
\subsection{Value across successive revisions}\label{sec:dependency-study}

These cases make revision a joint problem of restoration and preservation: improving a current plan can remove support needed by another project. To study this trade-off over successive changes, we vary edit scope, checker feedback and retained repair obligations, measuring loss as the fraction of artifacts inconsistent with current sources and maintenance rules.

We use six authored workspaces distinct from those used for policy development, each with four projects and two source changes. Three repeated trajectories per workspace compare DeepSeek Flash and GLM-5.3-Flash. The design crosses full-context or dependency-scoped editing with self-review, checker feedback, or feedback plus retained unresolved IDs. All branches receive the same consistency instructions and permit one call at the first event and up to two at the second, leaving initial failures available for the carryover comparison. Identical first-event prompts within a scope share one sampled candidate, so feedback and carryover branches start from the same state; different repeats and readers use separate draws. A one-pass reference takes the first candidate at each event.

The checker evaluates maintenance predicates against current sources and candidate parent values and returns failed item IDs. Feedback can expand the edit scope and trigger another call, while carryover makes unresolved IDs available before the next event's first proposal. Both readers use the same frozen indices, built by DeepSeek in three separate draws per workspace. The policies and analysis are fixed before revision outputs. The reader configurations expose different computation profiles: both use temperature zero, with output caps of 2,048 tokens for DeepSeek and 4,096 for GLM, and GLM reports 93,622 reasoning tokens despite the disabled-thinking request. These resources enter each reader's cost account.

\begin{table}[H]
\centering\small
\caption{Revision on six four-project workspaces, three repeated trajectories each. Loss is measured after the second event; valid counts are out of 18 trajectories. One-pass references use the first candidates at both events. The other policies have the same maximum allowance of one then two calls.}
\label{tab:revision-unseen}
\begin{tabular}{@{}llrrrr@{}}
\toprule
& & \multicolumn{2}{c}{DeepSeek Flash} & \multicolumn{2}{c}{GLM Flash} \\
\cmidrule(lr){3-4}\cmidrule(lr){5-6}
Scope & Revision policy & Loss & Valid & Loss & Valid \\
\midrule
Full & One pass & 0.128 & 11 & 0.000 & 18 \\
Full & Self-review & 0.188 & 9 & 0.000 & 18 \\
Full & Feedback & 0.114 & 11 & 0.000 & 18 \\
Full & Feedback + carryover & 0.142 & 11 & 0.000 & 18 \\
\midrule
Dependency & One pass & 0.056 & 15 & 0.028 & 15 \\
Dependency & Self-review & 0.056 & 15 & 0.028 & 15 \\
Dependency & Feedback & 0.056 & 15 & 0.000 & 18 \\
Dependency & Feedback + carryover & 0.083 & 15 & 0.000 & 18 \\
\bottomrule
\end{tabular}
\end{table}

The workspaces separate the roles of scope, feedback and execution (Table~\ref{tab:revision-unseen}). For DeepSeek, dependency-scoped self-review has 0.133 lower final loss than full-context self-review; the paired bootstrap interval is [$-0.266$, $-0.028$]. Full-context checker feedback reduces loss by 0.075 relative to self-review, with interval [$-0.166$, 0.000]. The one-pass reference also outperforms unassisted full-context self-review: extra computation can damage an already correct artifact. Including index construction, dependency self-review uses 85,020 input and 8,828 output tokens, compared with 83,537 and 7,631 for full-context self-review. Its lower loss is accompanied by higher total token use.

Index completeness localizes one source of lost value. The learned indices recover 414 of 426 required edge occurrences. Five workspaces have complete indices in every repeat; all dependency policies finish valid in all fifteen of their trajectories for both readers. In the build workspace, every repeat omits the same four source dependencies, leaving two stale items outside the first edit scope and imposing a minimum of two incorrect items out of twelve. GLM's first scoped proposals incur exactly those two errors, while DeepSeek's incur two additional errors inside the scope. DeepSeek's full-context first proposals, with all items editable, still contain two, seven and four errors across repeats. The missing entries thus explain a forced loss of repair access, with the remaining failures arising during execution. Dependency information directs edits across project boundaries while limiting exposure to further edits; Appendix~\ref{app:revision-confirmation} reports the complete-index and build strata separately.

Carrying failed IDs changes when repair occurs and how the reader uses the opportunity. For GLM's dependency policy, carryover repairs the three build trajectories in the second event's first call. Without carryover, those trajectories require the feedback call. Both reach zero final loss; carryover uses 36 instead of 39 revision calls and 474 instead of 510 local checks. Reader input falls while output rises, with index construction shared. This changes the correction-cost profile even when final value is equal.

DeepSeek's dependency carryover branch instead increases final mean loss from 0.056 to 0.083. In the build workspace, an unsuccessful repair reduces failed local checks from three to two while increasing globally incorrect items from four to six: downstream packaging is made consistent with an incorrect upstream product. Retaining the obligation preserves a repair opportunity; the reader's update determines whether that opportunity restores value or damages another use. This distinction motivates checking the support of a proposed update before committing it.

\subsection{Preservation and the price of validation}\label{sec:certified-commit}

Revision trades the value of useful new content against damage to existing uses and the cost of checking. Validating supporting information before writing an update provides a concrete instance of this trade-off. Consider a current requirement $x=0$ and a dependent value $y=1-x$. A memory $(x,y)=(1,1)$ has one incorrect item. The proposal $(1,0)$ makes the dependent equation locally consistent with the wrong $x$, increasing task error to two items. Rejecting the failed $x$ update while accepting the locally valid $y$ update leaves this damage in place. Requiring valid ancestors blocks it.

We compare four commit policies after a common proposal and local validation: \emph{accept all}; \emph{atomic acceptance}, which requires every check to pass; \emph{sequential local acceptance}, which rechecks each item against already committed parents; and \emph{ancestral acceptance}, which requires both its candidate check and all artifact ancestors to pass. All use the same complete dependency schema. Ancestral acceptance can retain valid partial progress when atomic acceptance rejects an entire proposal. For deterministic maintenance rules on an acyclic graph, both atomic and ancestral acceptance preserve every already correct item at a fixed source snapshot. Proposition~\ref{prop:certified-commit} derives the resulting correction value for any nonnegative item weighting; Equation~\eqref{eq:fallible-checks} bounds the expected new damage from false acceptance along an item's ancestry.

\paragraph{Holding the proposed revisions fixed.}
We replay all 244 distinct reader proposals from Section~\ref{sec:dependency-study}, each from its recorded pre-call state. Of DeepSeek's 130 proposals, accepting all introduces new incorrect items in 13 cases; sequential local acceptance does so in five. Atomic and ancestral acceptance produce none. Ancestral acceptance improves the starting loss on 67 proposals, compared with 48 for atomic acceptance. GLM's 114 proposals introduce no new damage under any policy; ancestral acceptance blocks some correct updates whose support remains uncertified, increasing mean committed loss from 0.013 to 0.018. Validation preserves existing value at the cost of sometimes delaying useful new content.

Rechecking retained values against the committed parents extends the available certificates when old support remains valid (Corollary~\ref{cor:retained-certificates}). This opportunity improves 2,024 of the 20,168 Boolean configurations with zero new damage. On the 244 saved model proposals, however, the additional checks leave committed states unchanged while costing DeepSeek 64 predicate evaluations. In particular, GLM's blocked updates depend on upstream values that are invalid in both the old and proposed states, so rechecking cannot supply the missing support. The benefit of retained certificates depends on which valid support survives, together with the cost of finding it.

\paragraph{Closed-loop comparison.}
A separate set of six four-project workspaces supplies two source changes and two repeated trajectories per workspace for both readers. All policies receive the same sources, maintenance notes, dependency lists and failed-check feedback, with at most two proposals per event. Only the commit rule changes. The complete checker schema is supplied by the environment in this experiment. The deterministic rules give an exact task target, allowing commit effects to be compared without ambiguity in the labels. Identical messages within a model, workspace, repeat and phase share one generated proposal; subsequent requests show the state actually committed.

\begin{table}[H]
\centering\small
\caption{DeepSeek commit policies on six additional four-project workspaces, two repeated trajectories each. Loss and valid counts concern the final state; damage counts commits that turn at least one correct item into an incorrect one under fixed sources. Predicate and edge checks are totals over both events.}
\label{tab:certified-commit}
\begin{tabular}{@{}lrrrrr@{}}
\toprule
Commit rule & Loss & Valid /12 & Damage & Predicates & Edges \\
\midrule
Accept all & 0.048 & 8 & 3 & 776 & 0 \\
Atomic & 0.190 & 8 & 0 & 790 & 0 \\
Sequential local & 0.024 & 10 & 0 & 1,332 & 0 \\
Ancestral & 0.024 & 10 & 0 & 902 & 381 \\
\bottomrule
\end{tabular}
\end{table}

Ancestral and sequential local acceptance reach the same final loss on these trajectories, using different validation resources (Table~\ref{tab:certified-commit}). Ancestral acceptance uses 430 fewer predicate evaluations and 381 additional Boolean parent checks. Relative to accepting all, it avoids 0.024 final mean loss; the paired bootstrap interval is [$-0.119$, 0.036]. The correction path also uses additional tokens and checks, whose separate prices enter the value--cost calculation in Appendix~\ref{app:certified-protocol}. Atomic rejection preserves correct items but leaves substantially more work unresolved.

GLM's first proposals are valid throughout this experiment, and all four policies reach zero final loss. Ancestral acceptance then adds 262 parent checks; sequential local acceptance adds 304 predicate evaluations. GLM uses the same 4,096-token cap and reports 16,835 reasoning tokens in this study. These outcomes locate validation within the risk--cost account: its realized benefit depends on the errors a reader makes, the useful updates it withholds and the price of checking.

The executable-rule reference reaches zero loss in 332 rule evaluations with no model calls. Its stored expressions directly generate valid states, whereas model policies receive prose rules and a candidate-checking interface. This makes the cost of directly executable knowledge visible alongside generated revision. Moving to the citation and plan cases in Section~\ref{sec:cases} adds a semantic decision: whether contextual evidence justifies a proposed edit. Extending the controlled comparison to those checks would connect their false acceptances, rejections and resource use to correction value. Appendix~\ref{app:revision-validity} gives the task-relative soundness condition and Appendix~\ref{app:certified-protocol} the resource account.

\section{Choosing what to retain, recover and revise}\label{sec:implications}

The value framework turns memory design into three connected decisions: what to preserve now, what to acquire when needed, and what to revise when evidence or goals change. Each allocates future decision capability across uses and resources. The following results give quantitative criteria for those choices. Allocation depends on both decision consequences and the cost of estimating them; recovery has value as an option exercised after a demand arrives; revision couples target improvement to the preservation of other uses.

\subsection{Allocate state and recovery by decision consequences}

Information acquisition and rational metareasoning choose further evidence or computation according to the improvement it can make to a decision, together with its cost \citep{hay2012computations}. Memory$^3$ already formulates memory-format allocation through writing cost and expected reuse-dependent reading cost \citep{yang2024memory3}. Bringing task loss and access failures into this account lets us compare retained content, a locator and direct archive search by their eventual decision consequences.

Consider one future use opportunity, with demand probability \(p\) and upfront costs \(c_{\rm keep}\) for content and \(c_{\rm idx}\) for a locator. Conditional on demand, retained content is used incorrectly with probability \(r_k\); recovery fails with probability \(q\), and successfully recovered content is used incorrectly with conditional probability \(r_r\). Failed recovery or incorrect use incurs loss \(L\). Let \(c_{\rm use}\) and \(c_{\rm read}\) be the conditional expected runtime costs of direct use and the full recovery-and-use route. With costs priced in the same loss unit, retained content has no greater expected loss-plus-cost when
\begin{equation}
c_{\rm keep}+p(c_{\rm use}+r_kL)
\leq c_{\rm idx}+p\bigl(c_{\rm read}+[q+(1-q)r_r]L\bigr).
\label{eq:allocation}
\end{equation}
The reliable-use special case sets \(r_k=r_r=c_{\rm use}=0\). Consequences, access reliability, reader use and price can each change the allocation. Appendix~\ref{app:allocation} derives the condition and gives a preference reversal caused by retained-content use failure.

Shared memory extends this accounting across future demands. A retained state or index incurs its construction cost once; over \(m\) uses, the upfront contribution to average cost is \(c_{\rm keep}/m\) or \(c_{\rm idx}/m\). Reading and revision costs accrue as the state is used. This amortization can make an initially expensive representation economical when it supports repeated decisions. The complete-split profiles in Figure~\ref{fig:public-profiles} expose this distinction directly: writing is spread over many questions in LoCoMo and charged to a single question per history in LongMemEval-S. Interactions between uses enter through their joint trajectory loss, preserving the complementarity illustrated in Section~\ref{sec:interaction}.

To choose an allocation, fix a demand law, a reader and a finite feasible family $\mathcal A$ containing a frequency-based allocation $a_F$. Let
\[
 J(a)=\E[\ell(\tau_a)]+\lambda^\top\E[C(\tau_a)]
\]
charge writing, maintenance and use at prices $\lambda\succeq0$ in the declared loss unit. The computation used to select an allocation has its own cost. Denote the opportunity beyond frequency allocation by $\Delta_F=J(a_F)-\min_{a\in\mathcal A}J(a)$.

\Needspace{9\baselineskip}
\begin{proposition}[Allocation with estimated value]\label{prop:allocation-regret}
Suppose $\sup_{a\in\mathcal A}|\widehat J(a)-J(a)|\leq\epsilon$ and the selected allocation satisfies $\widehat J(\widehat a)\leq\min_{a\in\mathcal A}\widehat J(a)+\zeta$. If selection costs are $c_H$ and $c_F$, then
\begin{equation}
 [J(\widehat a)+c_H]-[J(a_F)+c_F]
 \leq-\Delta_F+2\epsilon+\zeta+(c_H-c_F).
 \label{eq:allocation-regret}
\end{equation}
Consequently, $\Delta_F>2\epsilon+\zeta+(c_H-c_F)$ suffices for a strict reduction in total loss plus cost.
\end{proposition}

The margin separates available improvement from estimation error, optimization error and the extra work of making the choice. Appendix~\ref{app:allocation-policy} proves the bound and derives an allocation rule for separable items. In the plan example, let $w_Q,w_R,w_V$ weight Quince's implementation action, Rue's audit action and Verbena's joint release decision. Comparing a memory that retains only Quince's completion bit $A$ with one that retains only Rue's receipt bit $B$ gives $R^*(A)-R^*(B)=(w_R-w_Q)/2$. Changing the consequence of an audit error can therefore change which predicate is worth retaining at unchanged access frequency. Interacting facts remain in the joint objective $J(a)$. Under hard resource limits, selection and execution both enter feasibility; comparing the resulting choices across budgets produces the risk--cost profile.

\paragraph{Empirical implication (H1).}
A prospective comparison estimates allocation values on development histories and freezes its policy before future demands in evaluation histories are revealed. Varying error stakes and access or use reliability independently of frequency tests whether the predicted changes in allocation occur. Frequency allocation, direct recovery and complete-history use supply concrete alternatives. The comparison measures predicted and realized value together with estimation and selection costs, including regimes where the opportunity is too small to pay for the more elaborate choice.

\subsection{Preserve options for demands revealed later}\label{sec:recovery-choice}

Irreversible-decision theory values the ability to respond to information arriving after an initial commitment \citep{arrow1974irreversibility}. A retained route to an authorized source gives memory such an option: the agent can decide whether to acquire evidence once it sees the demand. Versioned evidence, stable locators and rollback add this flexibility at a storage and maintenance cost.

Consider one future decision under a fixed exogenous law. The agent first observes $W$, containing its memory, demand and lawful side information. It can pay conditional expected acquisition cost $c(W)$ to observe $Z$, including any failure signal. In a finite decision problem, let $r_0(w)$ be the minimum expected terminal loss before acquisition and $r_Z(w)$ the expected minimum after observing $Z$; action-dependent terminal costs are included in that loss. Define $G(w)=r_0(w)-r_Z(w)$. The post-acquisition policy can retain the original action, so $G(w)\geq0$.

\begin{proposition}[Value of a priced recovery option]\label{prop:recovery-option}
If keeping the path preserves the original state and actions, costs $c_0$ upfront, and permits optional acquisition after observing $W$, its net value is
\begin{equation}
 V_{\rm path}=\E[(G(W)-c(W))_+]-c_0.
 \label{eq:recovery-option}
\end{equation}
An optimal continuation acquires when $G(W)>c(W)$. The path is worth retaining when its expected exercised value exceeds $c_0$.
\end{proposition}

The result prices the option before the demand is revealed and the read after its value becomes apparent. It follows by choosing between $r_0(w)$ and $c(w)+r_Z(w)$ at each observation; Appendix~\ref{app:recovery-option} gives the proof, a state-capacity lower bound and the cost of content displaced by a locator. With implemented readers, the comparison uses the actual conditional improvement and resource use of the available continuation policies.

In the optimal-reader account of the plan example, a receipt read is worth $1/8$ when Quince is incomplete and $3/8$ when it is complete. After constructing the common base state, let the read cost $1/4$ and the additional upfront recovery-route price be $1/32$, including its archive, locator and reserved storage allowance. Reading only in the second case gives task risk $1/16$ and continuation loss plus cost $7/32$, improving on the no-path objective $1/4$ by $1/32$. Reading every time gives $9/32$. Conditional acquisition concentrates the expense on continuations where the receipt changes a more consequential decision. Source failure changes both the read decision and whether the path should be retained at all, as the same calculation in Appendix~\ref{app:worked-value} shows.

\Needspace{8\baselineskip}
\paragraph{Empirical implication (H2).}
The corresponding test reveals a new objective after writing, estimates its recovery value from separately calibrated evidence, and then chooses whether to acquire. Always recovering and directly rereading history provide alternatives at declared lifetime budgets. After any persistent repair, the original uses are evaluated alongside the changed target. The resulting observations test the predicted acquisition thresholds and place recovery policies on the correction-cost set in Equation~\eqref{eq:frontier}; source expiry, higher read prices and displaced content change which choices are worthwhile.

\subsection{Revise along evidence dependencies and preserve validity}

Truth-maintenance systems retain justifications and revise dependent beliefs when their supporting assumptions change \citep{doyle1979truth}. Incremental computation also preserves validity information: Adapton marks dependencies dirty and propagates changes to restore consistency when results are demanded \citep{hammer2014adapton}. For agent memory, this motivates retaining the source, object, version and reason behind a claim, together with unresolved checks from earlier revisions. A changed premise can affect claims and plans across project boundaries.

\paragraph{Revision principle: preserve opportunities and evaluate realization.}
Dependency records identify the uses affected by a change; unresolved checks retain obligations that later revisions must still address. Let $C_t$ be changed sources and $D_{t-1}$ contain the previously failed maintenance checks. The revision scope becomes
\[
 S_t=\operatorname{Desc}^{*}(C_t\cup D_{t-1})\cap V_{\mathrm{artifact}}.
\]
With conservative dependencies and complete recording of failed local checks, editing inside this scope preserves local validity outside it (Proposition~\ref{prop:dirty-scope}). This retains access to unfinished repairs. A commit rule addresses a second question: which proposed changes preserve already correct content? Under the functional maintenance model, atomic and ancestral acceptance provide that preservation; ancestral acceptance can also admit useful progress on independently certified branches (Proposition~\ref{prop:certified-commit}).

The measured value of these mechanisms includes the reader's execution and the price of checking. Retained obligations let GLM complete the same repairs with fewer calls and checks, while DeepSeek uses some opportunities to propagate an upstream error. Ancestral validation prevents that form of damage under the maintenance model but can withhold useful updates whose support remains uncertified. Rechecking retained support can release additional correct commits when its justifications survive; its value must therefore be assessed together with the search and checking cost. These outcomes connect restored target value, preservation of other uses and resource use, motivating the allocation of checks by the expected loss they avoid.

\section{Limitations}\label{sec:limitations}

The public-benchmark profiles use English-language tasks, one query-blind DeepSeek writer, and fixed ledger and recovery budgets. Their measured values depend on these choices as well as the reader. Varying writers, languages and budgets would map how retention and recovery trade off across a broader range of deployments.

The controlled revision studies use authored workspaces, deterministic maintenance predicates and two source changes per trajectory. These exact targets make preservation and collateral damage observable. Applying the same analysis to natural-language maintenance requires calibrated semantic checks, with false acceptances, rejections and checking costs measured together. The longitudinal cases add one owner's project experience; longer trajectories with natural goal changes and a wider range of users would extend the empirical account of memory's correction and preservation value.

\Needspace{17\baselineskip}
\section{Conclusion}\label{sec:conclusion}

Memory allocates decision capability across time. Its value begins with the distinctions available in historical information, continues through what the stored state preserves and the agent can use, and changes through recovery and revision. Expressing these stages in avoided task loss connects retention, access, computation and updating within a common decision-theoretic account. Risk--cost profiles then describe the resources required to support different uses under changing demands.

The theoretical constructions show how timing, interfaces, reader capabilities and feedback shape these profiles. The allocation bound identifies when available improvement pays for estimation and selection; the recovery-option identity identifies when later access pays for its preservation and use. The complete-benchmark studies quantify complementary value from retention and recovery, the effect of reader interfaces, and the dependence of cost on reuse. The project studies extend this account to revision, where restoring one capability can preserve, delay or damage another and where checking has its own price. This provides a foundation for choosing what to preserve, which options to keep available, and how to adapt memory while sustaining its value across projects and sessions.

%% file: sections/worked-value-example.tex
\subsection{An exact plan-continuation example}\label{sec:worked-value}

Consider three linked projects from the plan-maintenance rule family used in Section~\ref{sec:certified-commit}. Quince may reuse a completed implementation; Rue may reuse an audit only when its receipt is current; Verbena can release their joint result when both requirements hold. A finite instance makes the value account explicit. Let $A$ indicate Quince's completion and $B$ indicate Rue's valid receipt, independently and equally likely to be zero or one. The other release requirements are satisfied. A continuation plan predicts $(A,B,AB)$, with loss
\[
 \ell((A,B),a)=\tfrac14\mathbf1\{a_Q\ne A\}
 +\tfrac14\mathbf1\{a_R\ne B\}
 +\tfrac12\mathbf1\{a_V\ne AB\}.
\]
An incorrect joint release decision receives half the total weight. All readers know these rules, weights and the four-history law; evidence access distinguishes their states.

The original history supplies both predicates. A writer retains $M=A$, omitting the audit receipt. Without history, the smallest risk is $3/8$: each project predicate has error $1/2$, while holding the joint release has error $1/4$. With $M=A$, Quince's action is determined and the smallest risk is $1/4$. Thus the history carries $3/8$ value, the memory retains $1/8$, and the writing deficit is $1/4$.

To expose realization separately, specify a reader that applies the release rules when both predicates are available, but emits the default plan $(0,0,0)$ when either is missing. Its risk remains $3/8$ with $M=A$, although that state contains enough information to choose Quince's action correctly. Its implementation regret at $M$ is $1/8$, while its incremental realized value is zero (Table~\ref{tab:worked-value}).

\begin{table}[H]
\centering\small
\caption{One plan's value across information states, under the declared finite law and reader. Risks concern decisions made from the supplied state. Producing $M^+$ requires an archive read.}
\label{tab:worked-value}
\begin{tabular}{@{}lrrrr@{}}
\toprule
Supplied state & $R^*$ & $R_\pi$ & $V^*$ & $V_\pi$ \\
\midrule
Empty & $3/8$ & $3/8$ & $0$ & $0$ \\
Original history $H=(A,B)$ & $0$ & $0$ & $3/8$ & $3/8$ \\
Saved memory $M=A$ & $1/4$ & $3/8$ & $1/8$ & $0$ \\
After receipt recovery $M^+=(A,B)$ & $0$ & $0$ & $3/8$ & $3/8$ \\
\bottomrule
\end{tabular}
\end{table}

Reading Rue's archived receipt restores the missing predicate. A state that stores both predicates again supports the complete plan, and the specified reader also uses it correctly. This same example will distinguish recovery from later source revision in Section~\ref{sec:revision}, and value from the price of exercising that recovery option in Section~\ref{sec:recovery-choice}. Appendix~\ref{app:worked-value} gives the enumeration and the per-use preservation calculation.

%% file: sections/public-study.tex
\subsection{Retention and recovery across tasks}\label{sec:locomo}\label{sec:public-scale}

We evaluate three complete public splits: all 1,986 questions in LoCoMo's ten conversations, all 500 cleaned LongMemEval-S episodes, and all 589 questions in PersonaMem-v1's 32k split, covering twenty personas \citep{maharana2024locomo,wu2025longmemeval,jiang2025personamem}. These tasks combine factual recall, temporal and multi-session reasoning, knowledge updates, and responses to evolving preferences. LoCoMo histories contain 19--32 sessions; LongMemEval-S episodes contain 38--62. PersonaMem uses each question's released history cutoff, yielding 222 distinct prefixes.

A query-blind DeepSeek Flash writer constructs a ledger of at most 4,096 UTF-8 bytes for each history. We compare five information conditions: empty memory, ledger, question-based BM25 recovery, ledger plus recovery, and complete history. Recovery returns whole turns and adjacent context within 4,096 bytes. Both readers---DeepSeek Flash and GPT-4.1-mini---receive the same saved states and retrieval packets. The complete histories remain untruncated. LoCoMo additionally crosses every information condition with short-answer and structured answer-and-source instructions. One draw per cell produces 50,610 readouts over 3,075 questions, with the comparisons and scoring rules fixed before reading.

Here the intervention is evidence access under a common ledger and retrieval setup. Comparing complete systems such as Mem0, A-Mem and MemGPT would additionally evaluate their extraction, organization, updating and control policies \citep{chhikara2025mem0,xu2025amem,packer2023memgpt}.

Let $s_i$ be the task score and $\ell_i=1-s_i$. Paired score differences then measure realized value in avoided loss, while the complete-history arm measures the reader's realization from $H$. LoCoMo uses its released lexical QA score, LongMemEval its released semantic answer checks, and PersonaMem option accuracy. We report question-weighted means and 95\% bootstrap intervals, resampling whole conversations, episodes or personas while retaining paired conditions. Appendix~\ref{app:full-benchmarks} specifies the inputs, metrics and resource account.

\begin{table}[H]
\centering\small
\caption{Scores on the complete benchmark splits under the task-adapted reader instructions, using the released endpoint metrics. LoCoMo uses short-answer instructions. Both combines the ledger with the same recovery packet. Each cell averages every question in its split.}
\label{tab:public-scores}
\begin{tabular}{@{}llrrrrr@{}}
\toprule
Dataset & Reader & Empty & Ledger & Recovery & Both & History \\
\midrule
LoCoMo & DeepSeek & 0.167 & 0.397 & 0.524 & 0.552 & 0.535 \\
LoCoMo & GPT-4.1-mini & 0.225 & 0.446 & 0.556 & 0.590 & 0.538 \\
\addlinespace[2pt]
LongMemEval-S & DeepSeek & 0.080 & 0.388 & 0.410 & 0.598 & 0.668 \\
LongMemEval-S & GPT-4.1-mini & 0.098 & 0.450 & 0.418 & 0.642 & 0.662 \\
\addlinespace[2pt]
PersonaMem 32k & DeepSeek & 0.457 & 0.581 & 0.569 & 0.596 & 0.626 \\
PersonaMem 32k & GPT-4.1-mini & 0.365 & 0.559 & 0.503 & 0.598 & 0.593 \\
\addlinespace[2pt]
\bottomrule
\end{tabular}
\end{table}

Saved state and archival access contribute jointly to realized value (Table~\ref{tab:public-scores}). On LongMemEval-S, recovery raises the ledger score from 0.388 to 0.598 for DeepSeek and from 0.450 to 0.642 for GPT-4.1-mini. The paired gains are 0.210 [0.164, 0.256] and 0.192 [0.150, 0.234]. Conversely, adding the ledger to the identical retrieval packets raises performance by 0.188 and 0.224. The two sources therefore support complementary decisions even though retrieval itself is independent of the ledger. On LoCoMo, recovery gains are 0.155 and 0.143, positive in every conversation for both readers. PersonaMem has smaller gains of 0.015 and 0.039, with intervals [$-0.020$, 0.054] and [0.002, 0.074]. Appendix~\ref{app:full-benchmarks} reports the paired contrasts.

The amount of value already available without history also matters. On PersonaMem, the no-history scores are 0.457 and 0.365; the ledgers add 0.124 and 0.194 avoided loss beyond the question and the reader's existing capabilities. Complete history reaches the largest mean scores on LongMemEval-S. Within LoCoMo, it has a mean advantage of 0.019 and 0.017 over ledger plus recovery on the 1,540 answerable questions, while the adversarial questions reverse the aggregate ordering (Table~\ref{tab:locomo-categories}). Task composition therefore changes the realized-value comparison.

Reuse determines how construction enters the resource account (Figure~\ref{fig:public-profiles}). On LoCoMo, DeepSeek scores 0.552 with ledger plus recovery and 0.535 with complete history, using 2,167 versus 32,040 reported input tokens per read---a 6.8\% input ratio. Spreading the ten ledger writes over 1,986 questions adds about 158 input tokens per question, bringing the combined count to 2,326, or 7.3\% of the history route. LongMemEval-S has one question per history: constructing its ledger costs about 113,740 input tokens before a 2,003-token read. Its combined account is 115,743 input tokens per question, close to the 113,708-token full-history route. Thus the reuse horizon changes the resource trade-off alongside the task-dependent score profile. Output tokens, archival access and cache usage supply further cost coordinates (Table~\ref{tab:full-resources}).

\begin{figure}[t]
\centering
\includegraphics[width=\linewidth]{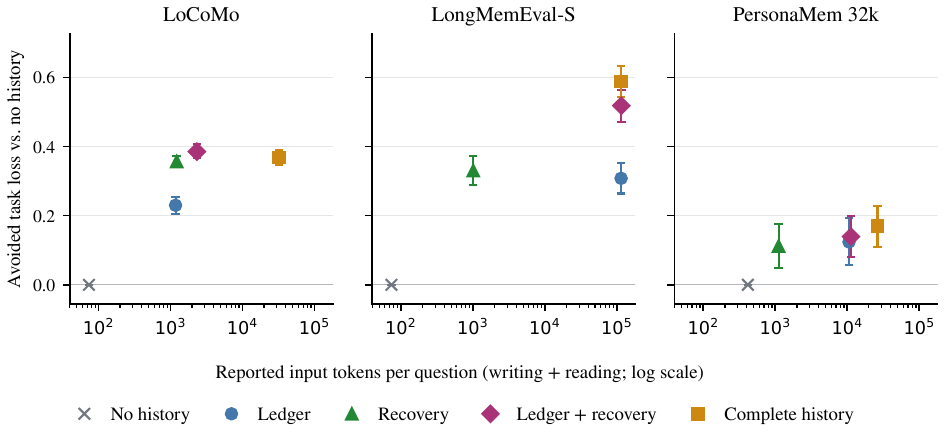}
\caption{Realized value and reported input cost for DeepSeek under the task-adapted interface. Value is the paired score gain over no history; bars give 95\% bootstrap intervals. Input counts include one writer pass per saved state, allocated across its released questions, plus the observed reads. Each panel uses its dataset's endpoint metric. Output tokens and other resource coordinates are reported separately.}
\label{fig:public-profiles}
\end{figure}

\subsection{Reading interfaces and measured value}\label{sec:full-interface}

The full LoCoMo comparison holds evidence fixed while changing the instruction bundle. Short-answer instructions place the question after the evidence and request a phrase; structured instructions put the question first, request a JSON answer with source IDs and permit explicit abstention. The interaction compares how this change affects complete history relative to bounded recovery.

For DeepSeek, the full-history--recovery contrast moves from $-0.090$ under structured instructions to $+0.011$ under short-answer instructions. The interaction is 0.101 [0.078, 0.129] and positive in all ten conversations. GPT-4.1-mini has a smaller interaction of 0.017 [$-0.003$, 0.034], with full history below recovery under both bundles (Table~\ref{tab:full-interface}). The same evidence therefore supports different measured performance across reader configurations. These comparisons connect the information available in a state to the implementation that realizes its value.

\begin{table}[H]
\centering\small
\caption{LoCoMo full-history minus recovery score, holding their evidence fixed across instruction bundles. The interaction is the short-answer contrast minus the structured contrast; brackets give 95\% bootstrap intervals from conversation resampling.}
\label{tab:full-interface}
\begin{tabular}{@{}lrrr@{}}
\toprule
Reader & Structured & Short answer & Interaction [interval] \\
\midrule
DeepSeek & $-0.090$ & 0.011 & 0.101 [0.078, 0.129] \\
GPT-4.1-mini & $-0.035$ & $-0.018$ & 0.017 [$-0.003$, 0.034] \\
\bottomrule
\end{tabular}
\end{table}

The category decomposition in Table~\ref{tab:locomo-categories} shows positive instruction interactions on all four answerable types for both readers; adversarial questions contribute $-0.040$ and $-0.035$ to the overall means. The measured profile also reflects the loss attached to the output. LoCoMo's adversarial scorer recognizes specific unavailable-information phrases, while structured instructions permit \emph{I do not know}. Both the behavior elicited by an interface and its relationship to the scoring rule enter the comparison conditions in Section~\ref{sec:contract}. The consequence-based study below makes this distinction visible by evaluating fixed outputs against their action requirements.

%% file: sections/appendix.tex
\section{Assumptions, proofs and exact controls}\label{app:proofs}

This appendix specifies the finite decision problems underlying the main-text counterexamples. Encoders, decoders, interfaces and losses are shared parts of each construction. The stated quantities are evaluated with exact arithmetic.

\subsection{Decision signatures before a query is revealed}\label{app:signature}

\begin{proposition}[Zero-error signature bound]
Let \(\mathcal H\) and \(\mathcal Q\) be finite, and let query \(q\) require the deterministic answer \(t_q(h)\). A deterministic writer stores \(f(h)\) before \(q\) is revealed; the reader observes only \(f(h)\) and \(q\), with no recovery access. Define
\[
T(h)=(t_q(h))_{q\in\mathcal Q},
\qquad K=|\{T(h):h\in\mathcal H\}|.
\]
Exact correctness for every \((h,q)\) requires at least \(K\) memory states and is achievable with \(K\) states. A fixed-length representation needs \(\lceil\log_2K\rceil\) bits.
\end{proposition}
\begin{proof}
If \(T(h)\ne T(h')\), some query has distinct target answers. Merging these histories gives the reader the same input for that query, preventing correctness on both. Every distinct signature therefore needs a separate state. Conversely, an index of the signature and a shared lookup rule suffice.
\end{proof}

The specification and decoder are shared and free in this construction. For an \(n\)-bit history and arbitrary coordinate queries, \(K=2^n\). A writer that instead sees \(q\) and can reread the history stores only the requested bit. In both cases the reader receives the query \(q\) for free; the difference is its revelation to the writer.

\subsection{Information order and actual-reader limits}\label{app:information}

\begin{proposition}[Simulation and optional ignorance]
Let \(M=f(H)\), with common side information \(S\), demand law and bounded loss. Suppose the raw-history policy class can compute \(f\) and simulate every memory policy within its constraints, and the memory policy class can simulate every no-memory policy by ignoring \(M\). Then
\[
R^*(H)\leq R^*(M)\leq R^*(\varnothing).
\]
\end{proposition}
\begin{proof}
Compose any memory policy with \(f\) to obtain a raw-history policy with the same actions. Execute any no-memory policy by ignoring \(M\). Taking infima over the respective policy classes gives the two inequalities.
\end{proof}

Under these conditions, the decomposition in Equation~\eqref{eq:value-accounting} has a nonnegative residual \(D^*(M)\). Its components share the unit of the common loss, and their ordering depends on simulation closure within the allowed resources.

A fixed reader can conceal this potential-value difference. For a fair answer bit \(Y\), a constant-zero reader has error \(1/2\) with either \(M=Y\) or \(M=\varnothing\). Its output law is identical in both conditions although an unrestricted reader has incremental values \(1/2\) and \(0\). Inspecting the representation or changing the reader can add the information needed to distinguish these explanations.

\subsection{The pointer's risk--cost profile}\label{app:pointer}

Let \(J\) be uniform on \(N\) opaque addresses, independently of fair bits \(B_1,\ldots,B_N\). A read at \(i\) returns \((B_i,\mathbf1\{i=J\})\). The only route to an unobserved bit or its target status is this read interface. The answer is \(Y=B_J\).

After any sequence of misses, unread addresses remain exchangeable and the target bit remains fair. With at most \(k\leq N\) distinct reads, a policy hits the target with probability at most \(k/N\). Reading distinct addresses and stopping at the first hit attains
\[
R_k^*=\frac{1-k/N}{2},
\qquad
\E[C_k]=\sum_{i=1}^k\left(1-\frac{i-1}{N}\right)
=k-\frac{k(k-1)}{2N}.
\]
The observed non-target bits provide no information about the unread target bit. In particular, zero error requires hitting the target, and the expected read count is \((N+1)/2\).

For \(M=J\), \(\Pr(Y=1\mid J=j)=1/2\) for every \(j\), so \(I(M;Y)=0\). The pointer directs one read to the target and attains zero error. Here addresses are opaque and a free locator is unavailable. If this single fixed answer were the writer's only future objective, storing the answer bit directly would use less space than a multi-bit pointer. The example isolates access value; alternative future queries, semantic filenames and index construction define other contracts.

\subsection{The cost-aware simulation condition}\label{app:simulation}

Fix a finite horizon, an environment family \(\mathcal E\), and a bounded loss \(L\) on external interaction consequences \(Z\). Write \(\Pi_i(b)\) for policies under contract \(i\) whose nonnegative cost vector is componentwise at most \(b\), almost surely in every \(E\in\mathcal E\). For a nonnegative overhead vector \(\delta\), the condition is
\[
\begin{gathered}
\forall\,\pi_0\in\Pi_0(b)\quad
\exists\,\pi_1\in\Pi_1(b+\delta)\quad
\forall\,E\in\mathcal E,\\
Z_1\overset{d}=Z_0,
\qquad C_1\leq C_0+\delta\quad\text{under a coupling}.
\end{gathered}
\]
The simulator is one causal policy across the family; its observations and actions respect the information order of contract \(1\). An environment-specific choice made using unavailable knowledge would change the quantifier order. Matching externally consequential histories preserves any declared loss measurable on those histories, including the consequences of earlier actions for later observations.

The exact control uses two worlds, two output actions and costs in archive reads and conversion steps. It checks five priors, \(3^4=81\) loss matrices and four hard budgets, yielding 1,620 exact risk comparisons. A read costs \((1,0)\) in the original interface and \((1,1)\) in the simulated one. At budget \((1,0)\), their optimal errors are \(0\) and \(1/2\); adding the declared conversion allowance gives zero error for both. This illustrates the role of \(\delta\) in Proposition~\ref{prop:simulation}.

\subsection{Interaction and feedback calculations}

The risk table in Section~\ref{sec:interaction} has \(R(0,0)=R(1,0)=R(0,1)=1/2\) and \(R(1,1)=0\). Allocating successive reductions in memory-first order assigns \((0,1/2)\) to memory and reader; reader-first order assigns \((1/2,0)\) to the same components. The cross-difference is \(-1/2\), exposing their interaction.

For the feedback construction, policy \(i\) answers \(i\), so its risk is \(p_0\) for policy 0 and \(1-p_1\) for policy 1. The two feedback kernels in Table~\ref{tab:feedback} give risk pairs \((0,3/4)\) and \((3/4,0)\). Their mixtures yield a relative-value set spanning \([-3/4,3/4]\). Forced uniform replay identifies conditional response behavior while leaving that mixture unspecified.

\subsection{Persistent revision and feasible costs}

A finite example separates a true capacity trade-off from avoidable overwriting. Let independent fair bits \(A,B\) be the history. Initially a two-state memory stores \(A\). A revision must decode \(B\) exactly for all four histories while retaining at most two states and having no recovery access. Histories with \(B=0\) and \(B=1\) must occupy disjoint nonempty state sets. Each set therefore has one state, leaving no distinction about \(A\); its Bayes error is \(1/2\). Four states suffice for exact support of both bits.

Equation~\eqref{eq:frontier} defines the feasible repair-cost set for a target-risk threshold and preservation tolerance. The two-bit construction gives an explicit infeasibility example under a fixed state budget and absent recovery access.

If a removed paragraph fits back into the same state budget and restores the tested behavior, that intervention exhibits a feasible alternative to the observed deletion. It addresses the proposed capacity explanation for that state. The two-bit construction above describes the complementary case in which the state budget genuinely prevents simultaneous support.

\subsection{Revision scope and unresolved validity}\label{app:revision-validity}
Let a directed acyclic graph contain source nodes and artifact nodes. Each artifact $i$ has a fixed local validity predicate $\phi_i$ depending only on its value and its recorded parents. Let $D$ contain every artifact whose predicate fails before the new event, and let $C$ be the source nodes changed by that event. Define
\[
 S=\operatorname{Desc}^{*}(C\cup D)\cap V_{\mathrm{artifact}}.
\]
Here $\operatorname{Desc}^{*}$ includes the starting nodes. Additional dependency edges can enlarge $S$; the graph must contain every actual dependency.

\begin{proposition}[Preservation with unresolved obligations]\label{prop:dirty-scope}
If a revision changes artifact values only inside $S$, every artifact outside $S$ remains locally valid. If the predicates inside $S$ also pass after revision, the entire state satisfies the maintenance constraints.
\end{proposition}
\begin{proof}
An artifact outside $S$ is outside $D$, so its old predicate holds. None of its parents can be a changed source or a changed artifact in $S$: an edge from such a parent would make it a descendant of $C\cup D$. Its value and all inputs to its predicate are therefore unchanged, so the predicate still holds. If the predicates inside $S$ pass as well, every predicate holds.
\end{proof}
The changed-source closure is the special case $D=\varnothing$ with a previously valid state. Retaining failed checks allows the scope to cover unfinished earlier repairs. This use of dependency and validity metadata connects to truth maintenance and incremental computation, where dirty dependencies are retained until change propagation restores consistency \citep{doyle1979truth,hammer2014adapton}.

For $N$ current values $z_i$ and nonempty acceptance sets $A_i$ grading each coordinate separately, an oracle allowed to edit $S$ has minimum loss
\[
 L^*_{S,\mathrm{oracle}}(z)=\frac1N\sum_{i\notin S}\mathbf{1}\{z_i\notin A_i\}.
\]
Every invalid value outside $S$ must remain. Assigning an accepted value to each coordinate inside $S$ attains the bound. This separates errors forced by the edit scope from the remaining realized loss; the latter also depends on available evidence and execution.

\paragraph{What a validity check establishes.}
A predicate $\phi_i(s,\hat z_i,\hat z_{P_i})$ can test a proposed update without returning its replacement. To certify correctness relative to a declared coordinate target $z^*(s)$, it must satisfy the following soundness condition:
\[
 \bigl[\hat z_{P_i}=z^*_{P_i}(s)\ \text{and}\
 \phi_i(s,\hat z_i,\hat z_{P_i})=1\bigr]
 \ \Longrightarrow\ \hat z_i=z_i^*(s).
\]
Topological induction then certifies accepted values in the same way as the functional construction below. A weaker predicate establishes only the constraints it checks: for a one-node task with target zero, an integer-type check accepts a proposed one and permits an error. Citation, schema and semantic checks therefore need a task-relative account of soundness. The functional rule used in the controlled study gives an exact predicate and an executable reference for this account.

\paragraph{Certifying the values that are committed.}
Fix current source values $s$ and a complete acyclic dependency graph. Let $P_i$ be artifact $i$'s artifact parents. Each artifact has a deterministic total maintenance rule $z_i=g_i(s,z_{P_i})$, giving a unique consistent state $z^*(s)$ by topological evaluation. Let $z$ be the current state and $\hat z$ a proposed state. Its local check is
\[
 b_i=\mathbf1\{\hat z_i=g_i(s,\hat z_{P_i})\}.
\]
In topological order, define $\kappa_i=b_i\prod_{j\in P_i}\kappa_j$, with an empty product equal to one. Commit $z'_i=\hat z_i$ when $\kappa_i=1$, and retain $z_i$ otherwise. The certificate concerns the fixed source snapshot; later source changes require revalidation of the affected dependencies.

\begin{proposition}[Value preserved by ancestral acceptance]\label{prop:certified-commit}
For $w_i\geq0$ and $\ell_w(z)=\sum_i w_i\mathbf1\{z_i\ne z_i^*(s)\}$,
\begin{equation}
 \ell_w(z')-\ell_w(z)
 =-\sum_{i:\kappa_i=1}w_i\mathbf1\{z_i\ne z_i^*(s)\}\leq0.
 \label{eq:certified-value}
\end{equation}
Thus one accepted update preserves every already correct coordinate and improves all such loss profiles by the weight of the corrected items.
\end{proposition}
\begin{proof}
If $\kappa_i=1$, every artifact parent is certified. Induction in topological order gives $\hat z_j=z_j^*(s)$ for those parents. Since $b_i=1$, the deterministic rule gives $\hat z_i=z_i^*(s)$. Every committed coordinate therefore takes its correct value; other coordinates retain their old values. Summing the coordinate differences proves Equation~\eqref{eq:certified-value}.
\end{proof}
Evaluating the predicates and propagating the certificate bits are separate cost coordinates: one predicate evaluation per item and at most one Boolean test per artifact edge. This connects validation of individual generated results \citep{pnueli1998translation} to dependency-based preservation \citep{coetzee2012build,hammer2014adapton}. Sequential local acceptance alone can propagate an incorrect parent: for $x^*=0$ and $y=1-x$, the old state $(1,1)$ and candidate $(1,0)$ satisfy the child's local test while increasing the number of incorrect items. Atomic acceptance avoids this damage by rejecting the entire candidate; ancestral acceptance can additionally retain correct updates on independent branches.

\paragraph{Fallible checks enter operational risk.}
Let a checker return $\widehat b_j$, with $\Pr(\widehat b_j=1,b_j=0)\leq\alpha_j$. Form certificate bits from these returned values and commit by the same rule. For fixed nonnegative weights,
\begin{equation}
 \E\!\left[(\ell_w(z')-\ell_w(z))_+\right]
 \leq \sum_i w_i\min\!\left\{1,\sum_{j\in\operatorname{Anc}^{*}(i)}\alpha_j\right\}.
 \label{eq:fallible-checks}
\end{equation}
Here artifact ancestry includes the item itself. An incorrectly certified item has an incorrectly certified ancestor whose parents have correct values; that ancestor's local check is false but reported as passing. A union bound over the ancestry, followed by the weighted sum over newly incorrect items, proves the inequality. No independence between checks is required. For adaptively chosen repeated checks, the corresponding bounds must hold conditional on the preceding execution. False rejections affect attainable correction and cost through withheld updates; false acceptances enter the damage bound.

\paragraph{Certifying retained support.}
An old value can support a new update when it remains valid under the current sources. Visit items in topological order. If all committed parents are certified, test the candidate value against those parents: accept and certify it on success. Otherwise keep the old value, certifying it only if it passes against the same certified parents. When a parent is uncertified, retain the old item without certifying it. Certificates are tied to the current source snapshot and actual committed parent values. Write $z^R$ for this result and $z^A$ for candidate-only ancestral acceptance.

\begin{corollary}[Reuse of retained support]\label{cor:retained-certificates}
Under the functional assumptions of Proposition~\ref{prop:certified-commit},
\[
 \{i:z_i^A=z_i^*(s)\}\subseteq\{i:z_i^R=z_i^*(s)\},
 \qquad \ell_w(z^R)\leq\ell_w(z^A)\leq\ell_w(z).
\]
\end{corollary}
\begin{proof}
Topological induction certifies only correct values, whether proposed or retained; unaccepted items keep their old values. Now consider an item certified by candidate-only ancestry. Its candidate ancestors are correct. Inductively the retained-support procedure certifies these same ancestors, whose committed values equal their candidate values because the functional target is unique. The item's candidate check therefore passes against those parents. It is accepted. Every correct coordinate of $z^A$ is consequently correct in $z^R$, giving the inequalities for nonnegative weights.
\end{proof}
For $x^*=0$, $y=1-x$, old state $(0,0)$ and proposal $(1,1)$, candidate-only ancestry rejects both updates. Certifying the retained $x=0$ instead permits $y=1$. In the exhaustive Boolean control, retained support improves 2,024 of 20,168 configurations and matches the rest, with no newly damaged coordinate. Checking retained values adds a resource cost; the corollary compares task loss for the same old state and proposal.

Exact controls enumerate all ordered Boolean DAGs with one to three nodes, all node truth tables and all current/candidate assignments: 20,168 configurations. Ancestral and atomic acceptance never increase error; sequential local acceptance does so in 996 configurations. Ancestral acceptance has lower error than atomic acceptance in 5,584 configurations. A second enumeration over candidate states and arbitrary reported check bits verifies the false-acceptance implication used in Equation~\eqref{eq:fallible-checks} for 20,168 configurations. These constructions fix the source snapshot and use complete dependency records.

\subsection{Failure and measurement uncertainty}

With delivery-failure probability \(p_F\), conditional delivered risk \(r\), and failure loss \(\rho\), total expectation gives \(R=(1-p_F)r+p_F\rho\). A system's memory-dependent failure probability can reverse an improvement conditional on successful delivery. Simulation monotonicity applies when the simulator also respects these operational consequences.

For \(n\) fixed units with \(w\) known errors and \(u\) unresolved labels bounded in \([0,1]\), empirical loss lies in \([w/n,(w+u)/n]\). If baseline and memory-condition intervals are \([l_0,u_0]\) and \([l_1,u_1]\), their value difference lies in \([l_0-u_1,u_0-l_1]\). A delivery failure with a specified penalty, a deliberate abstention and an unresolved evaluator label enter this calculation according to their respective loss definitions.

\subsection{Operational risk bounds and implementation regret}\label{app:operational}

For any \(\pi\in\Pi\), the definition of an infimum gives
\(G_\pi(M)=R(\pi,M)-R^*(M)\geq0\).
Expanding the two incremental values yields Equation~\eqref{eq:regret-gap}.
Its right side can be negative. For example, let the target always be zero and let memory store zero. The optimal risk is zero with either memory or no memory. A fixed reader that answers one without memory and zero with it has \(V_\pi(M)=1\), \(V^*(M)=0\), and \(G_\pi(M)=0\). The incremental-value difference is therefore \(-1\), despite zero implementation regret with memory.

Reader comparisons and representation analysis supply complementary evidence about optimal risk. For a fair bit \(Y\), a constant-zero reader has risk \(1/2\) both when \(M=Y\) and when \(M\) is constant, although the unrestricted optimal risks are zero and \(1/2\), respectively. A capable reader distinguishes the retained capability in these two states and lowers the upper bound in Equation~\eqref{eq:panel-reference}. A lower bound follows from two equally likely histories with different binary answers whose lawful observations remain identical under every allowed recovery action: their zero--one risk is at least \(1/2\).

For a sampling guarantee, fix \(K\geq1\) admissible policies before evaluating \(n\geq1\) independent episodes from the declared target law, with loss in \([0,1]\). Let \(\widehat R_\pi(M)\) be their empirical mean losses. Hoeffding's inequality \citep{hoeffding1963probability} and a union bound imply, for \(0<\alpha<1\), with probability at least \(1-\alpha\),
\[
R^*(M)\leq
\min\!\left\{1,\ \min_{\pi\in\mathcal B}\widehat R_\pi(M)
+\sqrt{\frac{\log(K/\alpha)}{2n}}\right\}.
\]
Each true policy risk is at most its empirical risk plus the displayed radius on the simultaneous event; taking minima and applying Equation~\eqref{eq:panel-reference} proves the claim. All policies can be evaluated on the same episodes. A panel selected or tuned on these outcomes is evaluated on fresh episodes or covered by a uniform bound for the selection procedure.

For studies with several queries per history, a history-level sampling design makes this guarantee directly usable. Draw \(n\) independent history episodes from the target law, each containing its queries and reader outcomes. For history \(j\) with \(m_j\geq1\) queries, define
\[
\overline\ell_{\pi j}=\frac{1}{m_j}\sum_{t=1}^{m_j}\ell_{\pi jt},
\qquad
\widehat R_\pi=\frac1n\sum_{j=1}^{n}\overline\ell_{\pi j}.
\]
The per-history means lie in \([0,1]\), so the same bound applies with \(n\) histories even when queries within a history are dependent. The target risk averages histories equally and then averages their queries. Thus shared histories determine the aggregation and sampling design used to estimate memory value.

Subtracting simultaneously valid risk intervals gives Equation~\eqref{eq:value-bounds}; optional ignorance further restricts the value interval to its nonnegative part. Reader measurements and representation arguments can therefore be combined to sharpen the same value comparison.

\subsection{A priced allocation calculation}\label{app:allocation}

Let \(D\) denote demand, with \(p=\Pr(D)\). Define \(r_k=\Pr(F_k\mid D)\) for incorrect retained-content use, \(q=\Pr(F_{\rm rec}\mid D)\) for failed recovery, and \(r_r=\Pr(F_r\mid D,F_{\rm rec}^{c})\) for incorrect use after successful recovery. The recovery route's conditional failure probability is \(q+(1-q)r_r\) by total probability; \(r_r\) is already conditional on recovery success. Parameters on zero-probability branches contribute zero.

The one-opportunity model assigns loss \(L\) to failed recovery or incorrect use and zero loss to correct use or no demand. Upfront costs are \(c_{\rm keep}\) and \(c_{\rm idx}\); \(c_{\rm use}\) and \(c_{\rm read}\) are expected runtime costs conditional on demand, covering each entire route, including failed attempts. Under declared resource prices \(c=\lambda^\top C\), the two expected objectives are
\[
\begin{aligned}
J_k&=c_{\rm keep}+p(c_{\rm use}+r_kL),\\
J_r&=c_{\rm idx}+p\bigl(c_{\rm read}+[q+(1-q)r_r]L\bigr).
\end{aligned}
\]
Equation~\eqref{eq:allocation} is \(J_k\leq J_r\); equality is a tie. For other losses or fallback actions, replace the failure terms by each route's conditional expected task loss. A budget-constrained comparison can retain the resource vector instead of assigning scalar prices.

For a route \(a\in\{k,r\}\) shared across \(m\) subsequent uses, let \(c_a\) be its upfront cost, \(C_{a,t}\) the priced runtime cost of use \(t\), and \(\ell_{a,t}\) its task loss. With additive task loss,
\[
J_a^{(m)}=c_a+\sum_{t=1}^{m}\E[C_{a,t}+\ell_{a,t}].
\]
Dividing by \(m\) gives the upfront contribution \(c_a/m\). Linearity of expectation permits dependencies between uses. General task interactions are represented by replacing the sum of expected task losses with \(\E[L(\tau)]\), retaining the same resource account.

For the reliable-use case \(r_k=r_r=c_{\rm use}=0\), let \(c_{\rm keep}=1\), \(c_{\rm idx}=0.1\), and \(c_{\rm read}=0.02\). A frequent low-stakes use with \((p,L,q)=(0.9,1,0)\) favors the locator, whose total expected cost is \(0.118\). A rare high-stakes use with \((p,L,q)=(0.01,1000,0.2)\) favors retained content: the locator costs \(2.1002\). Discarding all support costs \(pL\), respectively \(0.9\) and \(10\). In the second case, changing only retained-use failure to \(r_k=0.2\) raises \(J_k\) from \(1\) to \(3\), reversing the preference while \(J_r\) remains \(2.1002\). These finite examples separate access frequency, consequences and actual use.

\input{sections/design-results-proofs}

\section{Diagnostic protocols}\label{app:experiments}

\paragraph{Design and data reuse.}
The complete-benchmark comparison fixes its inputs, conditions, decoding and scores before reading. The workflow comparisons use two separate sets of six workspaces beyond those used for policy development, with each comparison's policies and analysis fixed before model outputs. Fixed-proposal replay evaluates commit rules on the same saved candidates; closed-loop evaluation follows the state each rule actually commits. The Cedar/Birch consequence audit assesses existing outputs under the task's action requirements.

\subsection{Exact constructions}

Exact controls for value, signatures, recovery and comparison run locally without model access. Enumeration counts describe the checked finite parameter space.

\subsection{Complementary roles for public benchmarks}\label{sec:benchmarks}

Established benchmarks supply different observables for studying memory value. LoCoMo provides long multi-session conversations with question answering, event summarization and multimodal continuation \citep{maharana2024locomo}. PersonaMem adds dynamic user profiles and preference-aligned response selection \citep{jiang2025personamem}. LongMemEval examines extraction, multi-session and temporal reasoning, knowledge updates and abstention, and organizes memory systems into indexing, retrieval and reading stages \citep{wu2025longmemeval}. MemOps makes lifecycle operations explicit through operation traces, target binding and state transitions \citep{hao2026memops}. These task structures support complementary parts of the value account.

\begin{table}[H]
\centering\small
\caption{Public benchmarks as sources of evidence for the framework.}
\label{tab:benchmark-roles}
\begin{tabularx}{\linewidth}{@{}p{0.19\linewidth}XX@{}}
\toprule
Benchmark & Useful observables & Value question it can help operationalize \\
\midrule
LoCoMo & Multi-session history and annotated evidence & What does a saved state support, and what can additional retrieval recover? \\
LongMemEval & Updates, time-dependent facts and abstention & Does later use follow the valid state, and when is additional evidence needed? \\
PersonaMem & Evolving preferences and response selection & Does retained history support decisions aligned with the current user? \\
MemOps & Operation targets, traces and resulting states & What consequence should a memory operation have, and which other facts should persist? \\
\bottomrule
\end{tabularx}
\end{table}

LoCoMo, LongMemEval-S and PersonaMem supply the behavioral evidence for retention and recovery in Section~\ref{sec:locomo}. The operation-level annotations released by MemOps support the complementary measurement study in Appendix~\ref{app:memops}.

\input{sections/full-benchmark-methods}

\subsection{Fictional update study}

These writer and reader calls also request \texttt{deepseek-v4-flash-0731}, at temperature 0.2 with thinking disabled. Cedar's target state has four fields: completed audit, pending comparison, plan revision \texttt{r3}, and a resource check before paid execution. Birch's four fields identify its report revision, saved commit and two citation roles. Cedar actions should avoid repeating the audit, prepare the comparison locally, and keep paid execution gated. Birch actions should preserve the valid citation and propose the other correction to its owner.

The study uses two writer draws per update policy and two reader draws per state and task. Alongside the four updated memories, readers receive an empty state, the prior state, or a selected source reference, giving 28 reader units in the clarified protocol. The four updated states support all tested fact fields in both reads of both projects. The strict complete-artifact endpoint accepts 17/28; the post hoc consequence endpoint accepts 22/28. All five changes concern the \texttt{hold}/\texttt{reuse} disposition shown in Figure~\ref{fig:cedar}.

\subsection{Multi-project revision protocol}\label{app:revision-confirmation}

The six authored revision workspaces have four projects and 12--14 artifact items each: shared approval and test-suite changes; feature-cache lineage; regional failover with an offline consumer; live and historical evidence; linked build products with a signed fallback; and reviewer sign-offs with deployment handoffs. Each begins in a valid state and receives two source changes in different parts of the workspace. Three fresh DeepSeek indices are generated for each workspace before either event. Five structures have complete indices in all repeats. The build structure omits four of its 22 required source edges in every repeat: two receipt-side dependencies and two required-version dependencies. All 18 indices are syntactically valid, with no extra edges.

All policies use the same consistency instructions, explicit prior source snapshot, current authoritative sources and cached artifact values. Full-context branches can edit every artifact. Dependency branches receive the recorded affected closure and its ancestors; failed IDs supplied by feedback or carryover expand the closure. The checker evaluates every public local maintenance predicate against the candidate state and returns only failed IDs. Evaluator-only checks of self-review outputs are excluded from policy costs. The first event ends after one proposal even when a check fails; the second allows a proposal and one review or feedback repair. Self-review always uses its second allowed call, while checked policies stop on complete validity. Taking only the first candidate at both events yields the separate one-pass reference in Table~\ref{tab:revision-unseen}.

Identical phase-specific messages share a sampled result within a model, workspace, repeat and scope. Every carryover comparison starts its second event with identical artifact values and different retained metadata. Sharing never crosses models or repeats. Logical policy costs include shared prefixes; physical billing counts each unique request once. The run has 18 index calls, 130 DeepSeek and 114 GLM revision calls. Two logical DeepSeek branches share a schema-invalid proposal and retain their prior candidate states.

Revision requests use temperature zero and output caps of 2,048 tokens for DeepSeek and 4,096 for GLM. Both request disabled thinking; GLM reports 93,622 reasoning tokens within 105,746 completion tokens. The 18 index builds cost 23,139 input and 3,655 output DeepSeek tokens. Table~\ref{tab:revision-costs} keeps these construction costs separate from each reader's token account.

\begin{table}[!htbp]
\centering\small
\caption{Logical costs over 18 two-event trajectories per policy. Reader tokens retain each model's accounting. Every dependency row additionally incurs 23,139 input and 3,655 output DeepSeek tokens for the shared 18 index builds. One-pass rows are derived from the first candidates.}
\label{tab:revision-costs}
\begin{tabular}{@{}llrrrr@{}}
\toprule
Scope & Policy & Calls & Input & Output & Checks \\
\midrule
\multicolumn{6}{l}{\emph{DeepSeek Flash}} \\
Full & One pass & 36 & 55,681 & 5,395 & 0 \\
Full & Self-review & 54 & 83,537 & 7,631 & 0 \\
Full & Feedback & 43 & 66,853 & 7,080 & 560 \\
Full & Feedback + carryover & 43 & 66,878 & 6,936 & 560 \\
Dependency & One pass & 36 & 44,166 & 4,171 & 0 \\
Dependency & Self-review & 54 & 61,881 & 5,173 & 0 \\
Dependency & Feedback & 39 & 48,663 & 5,421 & 510 \\
Dependency & Feedback + carryover & 39 & 50,946 & 5,668 & 510 \\
\midrule
\multicolumn{6}{l}{\emph{GLM Flash}} \\
Full & One pass & 36 & 51,024 & 33,500 & 0 \\
Full & Self-review & 54 & 76,524 & 56,706 & 0 \\
Full & Feedback & 36 & 51,024 & 33,500 & 474 \\
Full & Feedback + carryover & 36 & 51,024 & 33,500 & 474 \\
Dependency & One pass & 36 & 40,797 & 27,520 & 0 \\
Dependency & Self-review & 54 & 57,243 & 39,840 & 0 \\
Dependency & Feedback & 39 & 44,835 & 31,450 & 510 \\
Dependency & Feedback + carryover & 36 & 42,783 & 32,318 & 474 \\
\bottomrule
\end{tabular}
\end{table}

\paragraph{Complete indices and the build workspace.}
Table~\ref{tab:index-strata} separates complete-index workspaces from the build case, where all carryover differences occur. Missing required-version edges exclude stale \texttt{f\_link} and \texttt{f\_product} from the first edit scope, forcing loss $2/12$. GLM incurs exactly these errors; DeepSeek also errs on editable \texttt{e\_product} and \texttt{e\_release}. Full-context DeepSeek proposals have two, seven and four errors across repeats, despite access to every item. GLM's full-context proposals are correct.

\begin{table}[H]
\centering\small
\caption{Final mean loss by index completeness. Each complete-index group contains fifteen trajectories from five workspaces; the build group contains three trajectories from one workspace.}
\label{tab:index-strata}
\begin{tabular}{@{}llrrrr@{}}
\toprule
& & \multicolumn{2}{c}{DeepSeek Flash} & \multicolumn{2}{c}{GLM Flash} \\
\cmidrule(lr){3-4}\cmidrule(lr){5-6}
Scope & Policy & Complete & Build & Complete & Build \\
\midrule
Full & Self-review & 0.126 & 0.500 & 0.000 & 0.000 \\
Full & Feedback & 0.081 & 0.278 & 0.000 & 0.000 \\
Full & Feedback + carryover & 0.081 & 0.444 & 0.000 & 0.000 \\
Dependency & Self-review & 0.000 & 0.333 & 0.000 & 0.167 \\
Dependency & Feedback & 0.000 & 0.333 & 0.000 & 0.000 \\
Dependency & Feedback + carryover & 0.000 & 0.500 & 0.000 & 0.000 \\
\bottomrule
\end{tabular}
\end{table}

\subsection{Paired repetition and task resampling}\label{app:confirmation-statistics}

The workflow bootstrap intervals are central 95\% quantiles from 20,000 empirical resamples. For revision, the analysis array has six workspaces, three repeats and six paired policies. The primary outcome is invalid-item fraction after the second event. Fixed-workload intervals resample whole repeated trajectories within each fixed workspace. Workspace intervals first resample six workspaces and then three paired trajectories within each selected workspace. Both events and all policy branches stay together. Table~\ref{tab:revision-contrasts} reports the DeepSeek contrasts. For GLM, dependency feedback reduces loss by 0.028 relative to self-review, with a workspace bootstrap interval [$-0.083$, 0.000]; its dependency--full self-review contrast is 0.028 with interval [0.000, 0.083]. All other GLM final-loss contrasts are zero on every observed trajectory. An invariant observed outcome has a zero-width bootstrap interval.

\begin{table}[H]
\centering\small
\caption{DeepSeek paired revision-loss contrasts. Negative values favor the first named policy. Workspace resampling retains both transitions and every shared branch within a trajectory.}
\label{tab:revision-contrasts}
\begin{tabular}{@{}lrrr@{}}
\toprule
Contrast & Mean & Fixed workload & Workspace resampling \\
\midrule
Feedback $-$ self-review, full & $-0.075$ & [$-0.119$, $-0.035$] & [$-0.166$, 0.000] \\
Carryover increment, full & 0.028 & [0.000, 0.069] & [0.000, 0.102] \\
Feedback $-$ self-review, dependency & 0.000 & [0.000, 0.000] & [0.000, 0.000] \\
Carryover increment, dependency & 0.028 & [0.028, 0.028] & [0.000, 0.083] \\
Dependency $-$ full, self-review & $-0.133$ & [$-0.161$, $-0.105$] & [$-0.266$, $-0.028$] \\
Dependency $-$ full, feedback & $-0.058$ & [$-0.082$, $-0.037$] & [$-0.172$, 0.028] \\
Dependency $-$ full, carryover & $-0.058$ & [$-0.084$, $-0.032$] & [$-0.179$, 0.028] \\
\bottomrule
\end{tabular}
\end{table}

The task and repetition resampling schemes retain the shared histories, events and branch prefixes that define the paired comparisons.

\subsection{Dependency-validated commit protocol}\label{app:certified-protocol}

The fixed-proposal replay uses every distinct physical reader call from the multi-project study in Appendix~\ref{app:revision-confirmation}: 130 DeepSeek and 114 GLM proposals. Each is applied from its recorded pre-call state under the current sources. Reused branch prefixes are counted once. All commit policies receive the complete checker dependency graph for this intervention. Replay compares the effects of accepting the same proposed values; it leaves all original trajectories unchanged. Mean committed loss for accept-all, atomic, sequential-local and ancestral policies is respectively 0.1397, 0.1366, 0.1270 and 0.1206 for DeepSeek; for GLM it is 0.0132, 0.0263, 0.0132 and 0.0175.

\paragraph{Retained-support replay.}
A further offline comparison applies the retained-support rule to the same 244 old states and proposals. All four earlier policy outputs are first reproduced. Candidate checks are shared; additional checks reuse a cached result only for the same item value and parent tuple under fixed sources. DeepSeek's mean loss remains 0.1206 and GLM's 0.0175, with zero new damage and no changed committed state. DeepSeek uses 64 additional predicate evaluations; GLM uses none. The procedure adds 1,609 and 1,467 cache lookups respectively, while artifact-parent tests and generation costs remain unchanged.

The GLM gap relative to accepting all comes from the build workspace's three first-event proposals. The old and candidate values of \texttt{f\_link} and \texttt{f\_product} are both stale. Two proposed downstream updates are correct but lack certified ancestry, so rechecking the old values supplies no valid replacement support. The Boolean example has reusable valid support; these proposals require another source of justification or an upstream repair.

The closed-loop experiment adds six authored workspaces with four projects and 13--14 artifacts: compiler/ABI receipts with fallback products; credential generations and operation scopes; live, alternative and historical evidence; feature and training lineage; region/API compatibility; and completed implementation work with changed acceptance criteria. Each has two predetermined source events and two repeated trajectories per reader. Initial states are valid. The environment supplies complete checker dependency lists alongside the same prose maintenance rules to all four policies. All artifact IDs are editable. A current-state check precedes each event's revision; a valid state needs no model call. Otherwise up to two proposals are allowed, and the next request receives the actually committed state and its failed checks. Gate choice never appears in the reader instructions.

Candidate checks evaluate every local predicate. Accept-all applies the proposal; atomic applies it only if every predicate passes. Sequential-local visits items in topological order and reevaluates each proposed value against the parents actually committed so far. Ancestral acceptance propagates candidate check results through the complete dependency graph. When the merged state differs from both the candidate and the old state, another full check supplies feedback. Predicate evaluations and Boolean parent tests are counted separately. This complete-checker interface supplies stronger dependency information than the learned routing indices in the preceding experiment.

The protocol, workspaces, commit rules, budgets and analysis are fixed before new outputs. Requests use temperature zero, output caps of 2,048 for DeepSeek and 4,096 for GLM, and disabled gateway thinking. GLM reports 16,835 reasoning tokens. Identical prompts share a proposal within a model, workspace, repeat, event and phase. The experiment makes 60 physical DeepSeek and 22 GLM calls; all deliver valid proposals. Table~\ref{tab:commit-costs} charges each policy for the calls it uses.

\begin{table}[!htbp]
\centering\small
\caption{Logical commit-policy costs over twelve two-event trajectories. Complete dependency metadata is supplied by the environment to every policy. Predicate evaluations include initial checks, candidate checks and merged-state rechecks.}
\label{tab:commit-costs}
\begin{tabular}{@{}llrrrrr@{}}
\toprule
Reader & Commit rule & Calls & Input & Output & Predicates & Edges \\
\midrule
DeepSeek & Accept all & 32 & 54,756 & 3,835 & 776 & 0 \\
& Atomic & 33 & 56,793 & 3,949 & 790 & 0 \\
& Sequential local & 32 & 54,880 & 3,930 & 1,332 & 0 \\
& Ancestral & 32 & 54,884 & 3,923 & 902 & 381 \\
\midrule
GLM & Accept all & 22 & 33,752 & 19,757 & 636 & 0 \\
& Atomic & 22 & 33,752 & 19,757 & 636 & 0 \\
& Sequential local & 22 & 33,752 & 19,757 & 940 & 0 \\
& Ancestral & 22 & 33,752 & 19,757 & 636 & 262 \\
\bottomrule
\end{tabular}
\end{table}

The primary outcome is final invalid-item fraction. Paired bootstrap intervals use 20,000 draws, resampling whole workspaces and then complete repeated trajectories within them. For DeepSeek, ancestral minus accept-all is $-0.0238$ with interval [$-0.1190$, 0.0357]; ancestral minus atomic is $-0.1667$ with interval [$-0.4286$, 0.0000]; ancestral minus sequential-local is zero on each final trajectory. Holding workspaces fixed and resampling repeats gives [$-0.0476$, 0.0000] for the first contrast. GLM's final losses are all zero in these observations.

An executable-rule reference evaluates the stored maintenance expressions in topological order at each event and obtains zero loss in 332 rule evaluations over the twelve trajectories. It accesses the expressions directly; the model policies receive their prose representation and a candidate-checking interface. This reference exposes the value of directly executable knowledge in a workflow where those expressions are available.

For an example of pricing the correction, compare ancestral acceptance with accept-all for DeepSeek. Their final empirical loss difference is $-1/42$. Per two-event trajectory, the mean resource increments are $128/12$ input tokens, $88/12$ output tokens, $126/12$ predicate evaluations and $381/12$ Boolean parent tests, with equal model-call counts. Given prices in loss units, their observed loss-plus-cost difference is
\[
 -\frac{1}{42}+\frac{128\lambda_{\rm in}+88\lambda_{\rm out}
 +126\lambda_{\rm pred}+381\lambda_{\rm edge}}{12}.
\]
For GLM, ancestral acceptance changes no observed final loss or model-call usage and adds 262 Boolean parent tests. These points instantiate the correction frontier with both avoided damage and validation cost recorded.

\FloatBarrier
\section{Measurement checks on a released lifecycle benchmark}\label{app:memops}

MemOps provides operation traces, target identities and post-operation states for examining memory behavior. We inspect two deterministic postchecks and a cross-setting pairing rule. These postchecks assess memory operations separately from final-answer scoring.

We audit all 80 published Forget fixtures in the repository's evidence-conversation directory.\footnote{\url{https://github.com/MemTensor/MemOps}} The fixtures contain 400 question pairs, each instantiated in adjacent and longitudinal settings. For every row, we insert its published expected answer into the response field and run the existing deterministic postchecks. Pair members have identical answers and postcheck results, so Table~\ref{tab:memops} counts each pair once.

\begin{table}[H]
\centering\small
\caption{Deterministic postcheck flags under gold-answer replay. The denominator is 400 question pairs from 80 fixtures. Semantic classification concerns the flagged diagnostic.}
\label{tab:memops}
\begin{tabularx}{\linewidth}{@{}p{0.22\linewidth}rrX@{}}
\toprule
Postcheck & Pairs & Fixtures & Inspection result \\
\midrule
Over-forgetting & 24 & 20 & Retained information is available; the deletion phrase refers to another target or operation. \\
Disclosure & 11 & 10 & The gold answer repeats a forgotten detail; task guidance conflicts with the disclosure rule. \\
\bottomrule
\end{tabularx}
\end{table}

The over-forgetting check searches a 120-character window around a retained term for deletion expressions, with a list of retention-wording exceptions. This window can cross a semantic boundary. In fixture D22, the gold answer says that a Marriott loyalty record was removed while the Hilton preference \emph{remains fully usable}. The check flags Hilton as forgotten. Replacing that phrase with \emph{remains available} changes the flag from one to zero while preserving the same facts. This post hoc wording control isolates a measurement effect.

The disclosure conflicts require a different response. Historical-summary and candidate-selection gold answers reproduce details of records that the user requested to delete. Those details trigger the stated disclosure rule. Accepting every gold answer would conceal the inconsistency; the task and deletion semantics must first be reconciled. In the candidate-selection case, the question itself supplies the removed detail as an option, so its appearance in an answer also leaves retention in persistent memory unidentified. The audit preserves these flags separately from the over-forgetting errors.

The cross-setting pairing control groups responses by source file and question pair, seeking an adjacent response from the same reader. When that reader is absent, it uses another candidate. A constructed pair with a correct adjacent answer from reader A and an incorrect longitudinal answer from reader B yields a delta of \(-1\), combining reader and horizon changes. Requiring the same reader leaves this pair unmatched.

%% file: sections/design-results-proofs.tex
\subsection{Allocation with estimated consequences}\label{app:allocation-policy}

\begin{proof}[Proof of Proposition~\ref{prop:allocation-regret}]
Choose $a^*\in\arg\min_{a\in\mathcal A}J(a)$. Uniform objective error and approximate optimization imply
\[
 J(\widehat a)\leq\widehat J(\widehat a)+\epsilon
 \leq\widehat J(a^*)+\zeta+\epsilon
 \leq J(a^*)+2\epsilon+\zeta.
\]
Subtract $J(a_F)$ and add $c_H-c_F$. The definition of $\Delta_F$ gives Equation~\eqref{eq:allocation-regret}.
\end{proof}

The bound concerns decision-relevant objective error. Marginal calibration alone can leave costly items ordered incorrectly. For a fixed finite family, $0<\alpha<1$ and empirical objectives formed from $N\geq1$ independent workspace observations in $[0,L_{\max}]$, a two-sided Hoeffding union bound gives
\[
 \sup_{a\in\mathcal A}|\widehat J(a)-J(a)|
 \leq L_{\max}\sqrt{\frac{\log(2|\mathcal A|/\alpha)}{2N}}
\]
with probability at least $1-\alpha$. The same workspaces can be shared across allocations. Policy construction and tuning require independent evaluation workspaces or a uniform bound covering their selection. For uncertain demand, the proof also applies to $J(a)=\sup_{P\in\mathcal P}J_P(a)$ when its estimates meet the stated uniform tolerance; the conclusion then concerns the declared worst-case objective.

For separable items, the one-use objectives in Appendix~\ref{app:allocation} give the net saving $b_i=J_{r,i}-J_{k,i}$ from storing content instead of its locator. With equal incremental sizes, additive objectives and at most $k$ content slots, keep up to $k$ items with the largest positive $b_i$. Replacing a chosen item $i$ by an unchosen item $j$ changes the objective by $b_i-b_j$, proving the ordering rule. Unequal sizes yield a budgeted allocation problem; interactions between items remain in the joint objective $J(a)$.

An exact two-item example illustrates the opportunity gap. Let demand probabilities be $0.9$ and $0.1$, error losses be $1$ and $100$, and locator failure probability be $0.5$. Content use is correct and route costs are equal. With one content slot, keeping the frequent item leaves expected loss $5$; keeping the high-stakes item leaves $0.45$. The available reduction is $4.55$ loss units before estimation and selection costs. Finite rational controls additionally check Equation~\eqref{eq:allocation-regret} in 1,674 configurations of true objectives, bounded estimation errors and approximate optimizer outputs.

\subsection{Recovery-option value and lost distinctions}\label{app:recovery-option}

\begin{proof}[Proof of Proposition~\ref{prop:recovery-option}]
Conditioned on $W=w$, the best no-acquisition objective is $r_0(w)$ and the best acquisition objective is $c(w)+r_Z(w)$. Choosing between them gives
\[
 J_{\rm path}=c_0+\E\min\{r_0(W),c(W)+r_Z(W)\}.
\]
Subtracting this expression from $\E r_0(W)$ and using $x-\min(x,y)=(x-y)_+$ proves Equation~\eqref{eq:recovery-option}. The post-acquisition decision can reproduce the no-acquisition action for every signal, which gives $G(w)\geq0$ under the stated action permissions.
\end{proof}

The expectation is taken at write time, before the later demand in $W$ is revealed. The result holds for each admitted demand law separately. For a finite panel of implemented continuation policies, replace $r_0$ and $r_Z$ by the conditional objectives of the routes being compared. Their realized improvement can have either sign; the same positive-part identity characterizes the best gate between those routes. Estimating the gate incurs the objective errors and selection resources accounted for in Proposition~\ref{prop:allocation-regret}. Exhaustive action selection over binary source-and-signal models checks the option identity in 1,875 rational configurations, including uninformative signals and positive upfront costs.

\paragraph{Storage displaced by a recovery path.}
Suppose the path consumes space previously available to the base state, changing it from $M_0$ to $M_1$. Let $J_0(M)$ be the best no-path objective from that state and define $d_B=J_0(M_1)-J_0(M_0)$. Apply the option calculation to $M_1$. The comparison against the original state is then
\[
 V_{{\rm path},B}=\E[(G_1(W)-c(W))_+]-c_0-d_B.
\]
This follows by inserting $J_0(M_1)$ between the two compared objectives. Both states obey the same storage budget; archive storage, maintenance and access have their declared prices.

\paragraph{A lost decision distinction.}
If two equally likely histories induce the same saved state and lawful observations but have different binary targets, any policy restricted to those observations has error $1/2$. Writing $t$ for its probability of choosing the second target, its risk is $t/2+(1-t)/2$. Further computation or a rewrite from the same observations preserves this indistinguishability. A path that reveals the history on success and otherwise returns an independent failure symbol achieves error $q/2$, where $q$ is its failure probability. With error loss $L$, any target $qL/2\leq\epsilon<L/2$ is attainable through this path and unattainable for the restricted no-path family. User clarification or unindexed archive search, when available, enter that family's acquisition choices at their own costs.

The argument also has a capacity form. Let $U_1,\ldots,U_n$ be independent fair bits, let $M$ have at most $2^b$ states with $b\leq n$, and reveal an independent uniform coordinate query $Q$ after writing. For the smallest mean prediction error $e^*$,
\begin{equation}
 e^*\geq h_2^{-1}(1-b/n),
 \label{eq:capacity-risk}
\end{equation}
where $h_2^{-1}$ is the inverse binary entropy on $[0,1/2]$. Indeed,
\[
 H(U_Q\mid M,Q)=\frac1n\sum_i H(U_i\mid M)
 \geq\frac1nH(U\mid M)\geq1-b/n.
\]
For a deterministic optimal decoder, the binary error $E$ determines $U_Q$ given $(M,Q)$, so $H(U_Q\mid M,Q)=H(E\mid M,Q)\leq h_2(e^*)$. This proves the bound. A recovery interface supplies further information after writing and carries the additional storage and access coordinates used in its comparison.

\paragraph{Recovery and preservation.}
A route can preserve old uses by retaining their state and routing them to the same policies and observations, while giving the changed target a separate recovered layer. Under a common exogenous environment, the old action distributions and risks then agree. More generally, if every no-path repair has a causal simulator with the same target and preservation risks and additional resource vector at most $d$,
\[
 \mathcal K_0(\epsilon,\eta)+d\subseteq\mathcal K_1(\epsilon,\eta).
\]
Apply the simulator to the repair witnessing each feasible point. A recovery point that achieves the required risks below the best feasible no-path cost gives a strict frontier improvement. These comparisons charge routing, stored support and validation as well as acquisition.

\subsection{The plan-continuation value calculation}\label{app:worked-value}

The example in Section~\ref{sec:worked-value} instantiates the completion, audit and joint-release dependencies of the plan workspace. Its four equiprobable histories, weighted loss, fallback reader and prices are specified for exact calculation. Enumerating all eight binary action vectors gives
\[
 R^*(\varnothing)=3/8,\qquad R^*(H)=0,\qquad R^*(A)=1/4.
\]
Conditional risks from $M=A$ are $r_0(0)=1/8$ and $r_0(1)=3/8$. An exact receipt read reduces either conditional risk to zero. The specified incomplete-state fallback has conditional risks $1/8$ and $5/8$, so $R_\pi(A)=3/8$ and $G_\pi(A)=1/8$. Its empty-state regret is zero. These values yield every entry of Table~\ref{tab:worked-value}.

The common base-state construction cost cancels between the compared routes. At acquisition price $c=1/4$ and additional upfront path price $c_0=1/32$, the optimal-reader gate acquires only when $A=1$. The resulting task risk is $1/16$, expected acquisition cost is $1/8$, and total objective is $7/32$. The no-path objective is $1/4$, so the net option value is $1/32$. Reading on every continuation instead gives objective $9/32$. The same selective gate lets the specified fallback reader achieve risk $1/16$; against that reader's $3/8$ baseline its net improvement is $5/32$, reflecting both recovered evidence and improved realization.

If a read fails independently with probability $q$ and failure reveals no source information, its optimal-reader conditional gain is $(1-q)r_0(A)$. With the same prices, acquisition at $A=1$ is worthwhile for $q<1/3$, while retaining the path in advance has positive net value only for $q<1/6$. In the latter region the net value is $(1-6q)/32$. Enumerating all four deterministic gates verifies these thresholds and the option formula across the declared price and reliability grid.

For revision, start from $M^+=(A,B)$ and let a new source notice establish the current audit predicate $B^\dagger=0$. The target plan is now $(A,0,0)$. Grading the old stored plan against the new source gives expected loss $1/4$. Updating only the audit predicate gives zero risk; resetting the implementation predicate as well gives risk $1/8$. If the target use combines the audit and joint-release errors with weights $1/3$ and $2/3$, its risk falls from $1/3$ to zero under either patch. Quince's separate implementation risk remains zero under the scoped patch and rises to $1/2$ under the reset. Thus only the scoped patch meets zero per-use preservation tolerance. The patch uses one field assignment; acquisition and verification retain their separate resource coordinates.

%% file: sections/full-benchmark-methods.tex
\subsection{Complete benchmark comparisons}\label{app:full-benchmarks}

\paragraph{Tasks and information boundaries.}
The study uses \texttt{locomo10.json}, the cleaned LongMemEval-S release, and PersonaMem-v1's \texttt{questions\_32k.csv} with its shared contexts. All questions in these splits are included. LoCoMo contributes 282 multi-hop, 321 temporal, 96 open-domain, 841 single-hop and 446 adversarial questions. LongMemEval-S contributes 500 episodes across six question types, including thirty abstention questions. PersonaMem contributes 589 questions across seven task types and twenty personas. Each PersonaMem question receives exactly the published prefix ending at its specified message index; its 222 distinct prefixes share 37 underlying context streams. We preserve message roles and text when serializing each stream. The unit used for PersonaMem resampling is the persona.

LoCoMo histories include turn IDs, speakers, timestamps, text and released image captions. LongMemEval histories include every supplied session and its timestamp, with role-tagged turns. Answer labels, evidence annotations and supplied summaries are excluded from writing, retrieval and reading. Query-based retrieval uses only the question text, including neither the reference answer nor the answer options. Identical packets are reused in recovery-only and ledger-plus-recovery conditions. No-history reads retain the task question and its ordinary answer choices.

\paragraph{Writing and recovery.}
The writer receives the complete history before the question is revealed and is asked for a compact factual ledger preserving events, updates, preferences, relationships, dates and identities. Its request has temperature 0.2, a 1,200-token output cap and the thinking option disabled. The saved state is the first visible generation, capped deterministically at 4,096 UTF-8 bytes at a line or word boundary. A filtered or undelivered write supplies an empty state. One saved state serves all readers and information conditions for the same history prefix. Across 732 states, 731 contain nonempty text; 503 are shortened at the byte or output boundary.

BM25 uses lowercase alphanumeric tokens, $k_1=1.5$ and $b=0.75$. Ranked turns bring their immediately adjacent turns from the same session or source stream, subject to a 4,096-byte cap on whole returned blocks. Selected blocks are restored to chronological order. Query-weighted mean recovered sizes are 4,057 bytes for LoCoMo, 3,996 for LongMemEval-S and 3,805 for PersonaMem. The corresponding full-history sizes are 111,610, 511,053 and 142,479 bytes. Rebuilding the packets with the original direct BM25 implementation reproduces all 3,075 packets exactly.

\paragraph{Readers and instruction bundles.}
Both readers use fresh requests with temperature 0.2 and a 450-token output cap. Their requested identifiers are \texttt{deepseek-v4-flash-0731} and \texttt{gpt-4.1-mini}. The DeepSeek request disables thinking. Five evidence conditions are crossed with both readers; LoCoMo also crosses them with two instruction bundles. All cells are randomly interleaved. Short-answer prompts adapt the released task instructions to common memory and record blocks. LongMemEval supplies the question date; PersonaMem requests an option after \texttt{<final\_answer>}. Structured LoCoMo prompts place the question first and request a JSON answer and a list of source IDs, permitting an empty list and explicit abstention. The comparison varies the complete instruction bundle. Every planned read remains in its condition's denominator. Outputs that fail the completed-response or schema requirement receive operational score zero.

\paragraph{Measurement.}
LoCoMo's released scorer uses stemmed token F1, a multi-answer variant for category~1, and the first semicolon-separated reference for category~3. Adversarial questions display the released distractor and an unavailable-information option in a deterministic shuffled order: the latter appears in position~a for 208 questions and position~b for 238. Exact option-only responses are mapped to their option text. The released adversarial rule accepts the phrases \emph{not mentioned} and \emph{no information available}. The overall lexical QA score averages all five categories. PersonaMem uses its released option extractor, including its fallback to the complete response.

LongMemEval uses the released question-type-specific grading prompts, including their temporal, update and abstention criteria, with \texttt{gpt-4o-2024-11-20} at temperature zero and a 32-token output cap. The judge sees the question, reference and answer without reader or evidence-condition labels. Identical question--answer pairs share a verdict. One additional bounded grading request resolves missing measurements after the main pass; all required LongMemEval judgments are available.

An auxiliary LoCoMo semantic grader was checked on 64 questions crossed with a reference answer, an unrelated answer and an abstention. Despite an explicit benchmark-answerability field, it accepted unrelated responses on all eighteen unanswerable control questions; it matched the other 174 controls. Single-answer assessments agreed with 477 of 512 corresponding batch verdicts. We retain this as a measurement diagnostic and use the released lexical endpoint for the LoCoMo value profiles.

\paragraph{Paired estimates and task resampling.}
For a paired score contrast $d_i$, the question-weighted estimate is $N^{-1}\sum_i d_i$. We draw 10,000 empirical resamples of whole LoCoMo conversations, LongMemEval episodes or PersonaMem personas, retaining every condition paired within a question. If $w_c^{(b)}$ is a cluster's multiplicity in resample $b$, the resampled contrast is
\[
 \widehat\Delta^{(b)}=
 \frac{\sum_c w_c^{(b)}\sum_{i\in c}d_i}
 {\sum_c w_c^{(b)}n_c}.
\]
The reported bootstrap intervals are the central 95\% empirical quantiles \citep{hesterberg2015bootstrap}. Equal-cluster means are reported separately. These units give ten conversation clusters, 500 episode clusters and twenty persona clusters; the model calls are paired measurements within them. Workflow trajectory resampling is specified in Appendix~\ref{app:confirmation-statistics}.

\begin{table}[H]
\centering\footnotesize
\caption{Paired value contrasts on the complete benchmark splits and 95\% bootstrap intervals. Recovery compares both sources with ledger alone; the ledger increment compares both sources with the identical recovery packet; the last column compares full history with both sources. LoCoMo rows use the released lexical metric.}
\label{tab:full-contrasts}
\begin{tabularx}{\linewidth}{@{}llXXX@{}}
\toprule
Dataset & Reader & Recovery gain & Ledger increment & History $-$ both \\
\midrule
LoCoMo & DeepSeek & \shortstack{$0.155$\\$[0.128, 0.179]$} & \shortstack{$0.029$\\$[0.013, 0.043]$} & \shortstack{$-0.018$\\$[-0.037, -0.0003]$} \\
\addlinespace[3pt]
LoCoMo & GPT-4.1-mini & \shortstack{$0.143$\\$[0.118, 0.165]$} & \shortstack{$0.034$\\$[0.023, 0.045]$} & \shortstack{$-0.052$\\$[-0.064, -0.040]$} \\
\addlinespace[3pt]
LongMemEval-S & DeepSeek & \shortstack{$0.210$\\$[0.164, 0.256]$} & \shortstack{$0.188$\\$[0.150, 0.228]$} & \shortstack{$0.070$\\$[0.018, 0.120]$} \\
\addlinespace[3pt]
LongMemEval-S & GPT-4.1-mini & \shortstack{$0.192$\\$[0.150, 0.234]$} & \shortstack{$0.224$\\$[0.182, 0.266]$} & \shortstack{$0.020$\\$[-0.026, 0.068]$} \\
\addlinespace[3pt]
PersonaMem 32k & DeepSeek & \shortstack{$0.015$\\$[-0.020, 0.054]$} & \shortstack{$0.027$\\$[-0.018, 0.081]$} & \shortstack{$0.031$\\$[-0.014, 0.075]$} \\
\addlinespace[3pt]
PersonaMem 32k & GPT-4.1-mini & \shortstack{$0.039$\\$[0.002, 0.074]$} & \shortstack{$0.095$\\$[0.061, 0.132]$} & \shortstack{$-0.005$\\$[-0.037, 0.027]$} \\
\addlinespace[3pt]
\bottomrule
\end{tabularx}
\end{table}

\begin{table}[H]
\centering\small
\caption{LoCoMo contrasts by question type under the released lexical endpoint. History $-$ both compares complete history with ledger plus recovery under short-answer instructions. Interaction is the short-answer history--recovery contrast minus its structured counterpart. The aggregate weights each category by its question count.}
\label{tab:locomo-categories}
\begin{tabular}{@{}lrrrrr@{}}
\toprule
& & \multicolumn{2}{c}{DeepSeek Flash} & \multicolumn{2}{c}{GPT-4.1-mini} \\
\cmidrule(lr){3-4}\cmidrule(lr){5-6}
Question type & $n$ & History $-$ both & Interaction & History $-$ both & Interaction \\
\midrule
Multi-hop & 282 & $0.051$ & $0.081$ & $0.046$ & $0.082$ \\
Temporal & 321 & $-0.027$ & $0.195$ & $-0.013$ & $0.038$ \\
Open-domain & 96 & $0.008$ & $0.041$ & $-0.034$ & $0.065$ \\
Single-hop & 841 & $0.028$ & $0.227$ & $0.025$ & $0.073$ \\
Adversarial & 446 & $-0.146$ & $-0.177$ & $-0.291$ & $-0.157$ \\
\midrule
Answerable (1--4) & 1540 & $0.019$ & $0.182$ & $0.017$ & $0.067$ \\
All questions & 1986 & $-0.018$ & $0.101$ & $-0.052$ & $0.017$ \\
\bottomrule
\end{tabular}
\end{table}

The 446 adversarial questions form 22.5\% of the workload. Their negative interactions offset part of the positive contribution from answerable questions. Removing their weight raises the observed interaction from 0.101 to 0.182 for DeepSeek and from 0.017 to 0.067 for GPT-4.1-mini. These are decompositions of the same saved readouts.

\paragraph{Resources and reuse.}
Writing is charged once per history prefix in each policy that uses a ledger. Readers consume the saved state on every relevant question; recovery adds one archive request and its returned bytes. Thus a policy's input-token account over its released question workload is the sum of its writer and reader input counts, with their model identities retained. For the DeepSeek profiles, writer and reader use the same requested model. Each policy that uses a ledger carries its construction cost once over the question workload. The reported totals also keep output tokens, cache reports, search, storage and measurement separate.

\begin{table}[H]
\centering\small
\caption{Reported construction and readout tokens. Writer counts are charged once per history state; readout counts are per question. Uses/state is the released question count divided by the number of saved states. All writes use DeepSeek. Input/output are separate resource coordinates.}
\label{tab:full-resources}
\begin{tabular}{@{}lrrrr@{}}
\toprule
& States & Uses/state & Writer input/state & Writer output/state \\
\midrule
LoCoMo & 10 & 198.60 & 31,427 & 1,050 \\
LongMemEval-S & 500 & 1.00 & 113,740 & 1,153 \\
PersonaMem 32k & 222 & 2.65 & 25,158 & 728 \\
\bottomrule
\end{tabular}
\medskip
\begin{tabular}{@{}llrrrr@{}}
\toprule
& & \multicolumn{2}{c}{Ledger + recovery / question} & \multicolumn{2}{c}{History / question} \\
\cmidrule(lr){3-4}\cmidrule(lr){5-6}
Dataset & Reader & Input & Output & Input & Output \\
\midrule
LoCoMo & DeepSeek & 2,167 & 9 & 32,040 & 10 \\
LoCoMo & GPT-4.1-mini & 2,128 & 8 & 31,342 & 9 \\
LongMemEval-S & DeepSeek & 2,003 & 72 & 113,708 & 76 \\
LongMemEval-S & GPT-4.1-mini & 1,964 & 71 & 112,520 & 61 \\
PersonaMem 32k & DeepSeek & 1,810 & 20 & 26,354 & 50 \\
PersonaMem 32k & GPT-4.1-mini & 1,782 & 96 & 26,029 & 77 \\
\bottomrule
\end{tabular}
\end{table}